\documentclass[pdflatex,sn-mathphys-num]{sn-jnl}% Math and Physical Sciences Numbered Reference Style
\usepackage{graphicx}%
\usepackage{multirow}%
\usepackage{amsmath,amssymb,amsfonts}%
\usepackage{amsthm}%
\usepackage{mathrsfs}%
\usepackage[title]{appendix}%
\usepackage{xcolor}%
\usepackage{textcomp}%
\usepackage{manyfoot}%
\usepackage{booktabs}%
\usepackage{algorithm}%
\usepackage{algorithmicx}%
\usepackage{algpseudocode}%
\usepackage{listings}%
\usepackage{silence}
\usepackage{subcaption} % Enable subfigure-environment
\usepackage{tikz}
\usetikzlibrary{spy}
\usepackage{amsmath}
\usepackage{amssymb}
\usepackage{bm}
\usepackage{natbib}

\usepackage{placeins}
\usepackage{afterpage}

\usepackage[algo2e,ruled,vlined,linesnumbered]{algorithm2e}

\newcommand{\NN}{\mathbb{N}}

\newcommand{\RR}{\mathbb{R}}

\newcommand{\R}{\mathbb{R}}

\newcommand{\vertL}{\:\: | \:\:}

\newcommand{\dia}{\textnormal{dia}}

\newcommand{\ST}{\mathcal{S}}
\newcommand{\RC}{\mathcal{R}}
\newcommand{\RO}{\mathcal{O}}
\newcommand{\SK}{\Sigma}
\newcommand{\SE}{\Gamma}

\newcommand{\abs}[1]{\lvert #1\rvert}

\theoremstyle{thmstyleone}%
\newtheorem{theorem}{Theorem}%  meant for continuous numbers
\newtheorem{proposition}[theorem]{Proposition}% 
\newtheorem{definition}{Definition}%

\theoremstyle{thmstyletwo}%

\theoremstyle{thmstylethree}%

\begin{document}

\title[Skeleton Sparsification and Densification Scale-Spaces]{Skeleton Sparsification and Densification Scale-Spaces}


\author[1]{\fnm{Julia } \sur{Gierke}}\email{gierke@mia.uni-saarland.de}
\author*[2]{\fnm{Pascal} \sur{Peter}}\email{peter@math.uni-sb.de}
\equalcont{These authors contributed equally to this work.}

\affil[1]{\orgdiv{Mathematical Image Analysis Group}, \orgname{Saarland University}, \orgaddress{\street{Campus E1.7}, \city{Saarbr\"ucken}, \postcode{66041},  \country{Germany}}}

\affil[2]{\orgdiv{Department of Mathematics and Computer Science}, \orgname{Saarland University}, \orgaddress{\street{Campus E2.4}, \city{Saarbr\"ucken}, \postcode{66041}, \country{Germany}}}


\abstract{
	The Hamilton-Jacobi skeleton, also known as the medial axis, is a powerful shape descriptor that represents binary objects in terms of the centres of maximal inscribed disks. Despite its broad applicability, the medial axis suffers from sensitivity to noise: Minor boundary variations can lead to disproportionately large and undesirable expansions of the skeleton. Classical pruning methods mitigate this shortcoming by systematically removing extraneous skeletal branches. This sequential simplification of skeletons resembles the principle of sparsification scale-spaces that embed images into a family of reconstructions from increasingly sparse pixel representations.
	
	We combine both worlds by introducing skeletonisation scale-spaces: They leverage sparsification of the medial axis to achieve hierarchical simplification of shapes. Unlike conventional pruning, our framework inherently satisfies key scale-space properties such as hierarchical architecture, controllable simplification, and equivariance to geometric transformations. We provide a rigorous theoretical foundation in both continuous and discrete formulations and extend the concept further with densification. By growing the skeleton successively instead of shrinking it, we allow inverse progression from coarse to fine scales. Densification scale-spaces can even reach beyond the original skeleton to produce overcomplete shape representations with relevancy for practical applications. 
	
	Through proof-of-concept experiments, we demonstrate the effectiveness of our framework for practical tasks including robust skeletonisation, shape compression, and stiffness enhancement for additive manufacturing.
	}




\maketitle

\section{Introduction}

Classical scale-spaces embed images into a family of hierarchically simplified 
versions of the original. Such image evolutions are typically based on partial 
differential equations (PDEs) \cite{AGLM93,Ii62,Li11,SchW98,We97} which remove 
information by smoothing, thus simplifying the image. 
\citet{CPW19} have introduced sparsification as an alternative: They create 
coarser scales by explicitly removing image pixels. At each scale, they 
interpolate the missing image parts to create a sequence of complete images. On 
one hand, these scale-spaces share many important properties with classical 
scale-spaces. On the other hand, they also open up new applications for 
scale-space theory such as image compression. 

In the field of shape analysis~\cite{SP08}, notions of hierarchical 
simplification that resemble sparsification are well-known in a different 
context. Shapes can be equivalently described by the medial axis introduced by 
\citet{Bl67} as a symmetry-based, central shape descriptor intended to mimic 
biological sensory processes. Due to its appearance as a thin set of connected 
lines and arcs centred in the shape, the medial axis is also referred to as the 
skeleton. Unfortunately, it is highly sensitive to perturbations on the shape 
boundary, which lead to unwanted branches in the skeleton. Therefore, so-called 
pruning techniques \cite{BLL07,ML12,Og94,SB98,SS16,TH02} have been proposed. 
They sequentially remove parts of the skeleton. The similarity to the 
sparsification processes of \citet{CPW19} suggests that there might be 
hitherto unexplored connections to scale-space theory.

This is surprising, since shape analysis has been known to be closely connected 
to scale-space theory. Significant research efforts have been dedicated to the 
exploration of morphological scale-spaces   
\cite{AGLM93,BM92,CS96,KS96,ST93,BS94}. Here, simplifying shape evolutions are 
derived from the basic erosion and dilation operators \cite{So99a}. Thus, in 
contrast to the skeleton view point, they operate on the shape outline instead 
of a central shape descriptor. While multi-scale ideas have been previously 
connected to the medial axis transform 
\cite{PSSDZ03,SBS16,SBS17,TK12,PBCFM94,PEFM98}, these focus on counteracting 
noise or on generalisations of the shape descriptor. In contrast, we aim to 
build a new scale-space theory on the combined concepts of skeletonisation and 
sparsification.

\subsection{Our Contribution}

We propose a scale-space framework that combines the ideas of image 
sparsification \cite{CPW19} and multi-scale shape analysis: Our \emph{shape 
sparsification} transfers classical scale-space properties \cite{AGLM93} such 
as hierarchical architecture, quantifiable simplification, and equivariances to 
the medial axis setting \cite{Bl67}. To this end, we consider sparsified 
skeletons instead of sparsified images and derive theoretical guarantees that 
apply to shape descriptors and their associated binary images. Furthermore, we 
complement the sparsification idea by a densification counterpart: By extending 
the original 
skeleton, we can reach finer skeleton scales that go beyond the initial scale 
corresponding to the original shape. With proof-of-concept examples, we 
illustrate how our theory supports a variety of 
different practical applications ranging from compression over skeleton 
robustification to pre-processing for 3-D printing. 

Beyond a more rigorous discussion of theoretical background and a more detailed review of related work, we extend our previous conference contribution~\cite{GP25} on skeletonisation scale-spaces in the following ways:
\begin{enumerate}
	\item We propose a fully continuous scale-space theory and discuss differences between continuous, discrete, and semi-discrete skeletonisation scale-spaces in more detail.
	\item With extended experiments for compression and branch pruning, we provide more in-depth quantitative results for potential practical applications.
	\item Densification adds a whole new paradigm to our framework. A 
	generalised bottom-up scale-space can not only grow a skeleton from a 
	single point, but also allows to add finer scales beyond the original 
	skeleton.
	\item We illustrate the usefulness of densification scale-spaces with a practical application for stiffness enhancement~\cite{BBYBP20}.
\end{enumerate}

\subsection{Related Work}

Research on the medial axis transform and scale-space theory provide the 
foundation for our work. We rely on classical definitions and fundamental 
properties from both fields. In the following we provide a short overview of 
the literature on skeleton pruning and thinning, as well as its image 
counter-part of sparsification and densification. These two fields inspire 
concrete implementations and practical applications of our scale-space 
framework.

\subsubsection{Medial Axis Transform}
\label{sec:ma}

As noted by \citet{hajlasz2020old}, the underlying concepts of the medial axis 
were already studied by \citet{erdos1946hausdorff} before the term itself was 
coined. \citet{Bl67} rediscovered the medial axis as an alternative to 
Euclidean geometry, which he deemed too specific, and topology, which he saw as 
too general. The idea of a central shape descriptor was motivated by the 
so-called grass-fire analogy, where prairie fires spread from the boundaries. 
The skeleton is formed by shock locations where these propagating fronts meet 
inside of the object. From these points, the object can be reconstructed 
\cite{GK03b}. The evolution of the propagating front in the grass-fire model is closely
related to the Eikonal equation, which can be solved by fast marching methods \cite{Se96} or in terms of viscosity solutions~\cite{RT92}. However, solving the Eikonal equation alone does not suffice to compute
the skeleton since the locations of the shocks need to be explicitly identified \cite{siddiqi:2002}.

Relations between the shape boundary and the medial axis have 
also been explored by \citet{KO05} as well as \citet{Da05}.
The medial axis is a subset of the more general symmetry 
sets~\cite{BGG85}, which are closely related, but not part of our theoretical 
framework. 
Therefore, we refer to \cite{PSSDZ03,SBS17,TDSAT16} for a more detailed review 
of this literature. 
For our purpose, classification of skeleton points into categories like end or branching points is crucial for our concrete task-specific scale-spaces. Here, we rely on the work of \citet{GK03}. 

Our work is not the first to connect skeletons to multi-scale ideas. In 
Section~\ref{sec:rel_pruning} we discuss pruning and thinning algorithms 
separately as prime examples that come closest to our own notion of skeleton 
sparsification. Additionally, \citet{PSSDZ03} and \citet{SBS17} also provide an 
overview 
of more broadly related methods that use PDEs for shock computation 
\cite{kimia1995} or use the concept of cores \cite{PEFM98}. The latter uses the 
notion of medialness, a relaxation of the strict binary conditions of the 
medial axis by observing evidence for points being approximately centred in a 
shape in a disk shaped neighbourhood. The radius of this neighbourhood defines 
a scale-parameter and cores are defined via traces in such a scale-space. Like 
similar concepts such as the multi-scale medial axis of \citet{PBCFM94} or the 
learning-based approach of \citet{TK12}, the focus differs from our own work: 
We develop a scale-space theory by explicitly applying sparsification 
techniques to the skeleton.

The medial axis transform has a wide range of practical applications that use 
either 2-D or 3-D \cite{TDSAT16} skeletons. Classically, the medial axis has 
been used 
in computer vision for tasks like object recognition and shape matching 
\cite{opiyo2021}, but there are also applications in physics 
simulations~\cite{li2023}, compression \cite{Mu20}, CAD modelling 
\cite{mayer2023computational} or 3-D printing \cite{BBYBP20}.

\subsubsection{Scale-Space Theory}

Classical scale-spaces as introduced by \citet{Ii62} rely on Gaussian blur to gradually remove image features with increasing scale. Such evolutions can be expressed by a linear partial differential equation (PDE).
This theory has been studied in great detail \cite{AGLM93,Ii62,Wi83,Li94,SNFJ96,Fl13} and has been extended to nonlinear \cite{SchW98,We97} or pseudodifferential evolutions \cite{DFGH04,SW16,FS01,BDW05}. 
Beyond linear and nonlinear diffusion processes, variational models have been studied in a scale-space context as well. The total variation (TV) flow is a particularly well-known alternative nonlinear evolution that has been studied extensively \cite{ABCM98,BCN02}. More general variational approaches to image and shape evolution have been considered for instance in shape space \cite{wirth2010variational}.
While we do not rely on PDEs or variational methods in this work and restrict ourselves to operations on skeletons, the core principles of classical scale-space theory are highly relevant. We build our theory on the architectural, invariance, and information reduction concepts introduced by \citet{AGLM93}.

Regarding our goal of constructing hierarchical shape evolutions, the class of morphological scale-spaces \cite{AGLM93,BM92,BS94,CS96,KS96,ST93} is closer to our on work than the aforementioned Gaussian category. The basic building blocks of morphological scale-spaces are dilation and erosion~\cite{So99a} operations. In a binary image, these supremum and infimum operations affect the boundary of the shape and thus naturally lead to boundary evolutions. Our framework shares similarities with some of these scale-spaces, for instance the hierarchical shrinking of objects for increasingly coarser scales. However, we define our evolution entirely on the skeleton as an equivalent shape descriptor instead of modifying the boundary.

\subsubsection{Skeleton Pruning and Thinning}
\label{sec:rel_pruning}

The sequential removal of skeleton parts has been studied for two major purposes: \emph{Thinning} addresses the fact that discrete skeletons obtained by practical methods of computation do not necessarily have the desired one-pixel width. \emph{Pruning} approaches focus instead of removing spurious skeleton branches that arise from noisy shape boundaries. Both can be used in tandem, typically thinning the skeleton first and applying pruning afterwards. Some algorithms attempt to address both goals simultaneously by integrating thinning and pruning criteria into a single framework \cite{zhu1994,palagyi2009,pudney1998,PB12}.

\textbf{Thinning:} At first glance, thinning methods 
\cite{zhu1994,palagyi2009,pudney1998,malandain1998,PB12} seem closely related 
to our sparsification since they remove individual pixels step-by-step. 
However, they are a method of discrete skeleton computation. Thus, they 
successively remove points from the original shape based on some criterion of 
medialness. This ensures that the final skeleton is thin (one pixel width). In 
contrast, we remove points from the already computed skeleton. However, for the 
computation of our initial skeleton in our experiments on discrete data, we use 
maximal disk thinning \cite{PB12}. It is a homotopic thinning method that 
preserves the connectivity of the initial shape and removes shape pixels 
ordered by their distance from the boundary. Skeleton end points define a 
stopping criterion for the algorithm and rely on the maximal disk 
identification method of \citet{remy2005}. 

\textbf{Pruning:} In contrast to thinning, branch pruning methods~\cite{BLL07,ML12,Og94,SB98,SS16,TH02} are not limited to the discrete setting. Small bumps or kinks in the object boundary can have a large impact on the skeleton, adding substantial branches both in the continuous and the discrete setting. Since pruning approaches remove parts of the skeleton, they are more closely related to our sparsification scale-spaces. We generalise the removal of skeleton points, but also show that classical branch pruning is one of the practical applications that is covered by our scale-space theory Section~\ref{sec:applications}. In particular, our experiments are inspired by the hierarchic branch pruning of Ogniewicz~\cite{Og94}.

\subsubsection{Sparsification and Densification}
\label{sec:rel_opt}

Our scale-space framework is inspired by two previous publications that successively remove image content to define a family of simplified images. \citet{CPW19} use spatial sparsification: They remove image pixels to transition from one scale to the next. These techniques originate from so-called spatial optimisation problems in image compression and sparse data representation \cite{MHWT12,HMHW17,DDI06}. Here, a set of optimal known data points has to be chosen for image reconstruction with interpolation methods such as PDE-based inpainting. Similar concepts have also been used for a sparsification of the co-domain. Quantisation scale-spaces~\cite{Pe21} reduce the number of different grey values to construct a multi-scale family of images. 

The opposite notion of densification \cite{CW21,DAW21,KBPW18} is well-known in spatial optimisation as well. Such approaches typically start with an empty image and greedily add pixels that are most important for the image reconstruction. Compared to sparsification, these approaches are often more robust to noise, since the initially chosen pixels have a global impact instead of the highly localised influence of individual pixels in a dense known data set. For our skeletonisation scale-spaces, we leverage both of the aforementioned paradigms.

\subsection{Organisation}
%\medskip
%\textbf{Organisation.}
Section~\ref{sec:skeletonisation} provides the theoretical background required for our new scale-spaces in Section~\ref{sec:skelscale}, which we adapt to specific applications in Section~\ref{sec:applications}. We conclude with a discussion and outlook on future work in Section~\ref{sec:conclusion}.

\section{Review: Skeletonisation}
\label{sec:skeletonisation}

\begin{figure}[t]
	\small
	\tabcolsep2pt
	\begin{center}
		\begin{tabular}{cccc}
			\textbf{(a) apple-8} & \textbf{(b) distance map} & \textbf{(c) overlap} & \textbf{(d) apple-9}\\
			\includegraphics[width=0.24\textwidth]{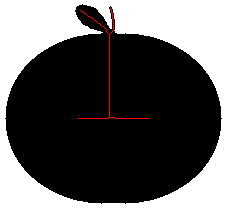} &
			\includegraphics[width=0.24\textwidth]{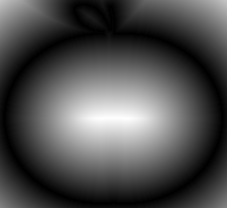} &
			\includegraphics[width=0.24\textwidth]{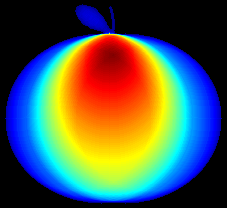} &
			\includegraphics[width=0.24\textwidth]{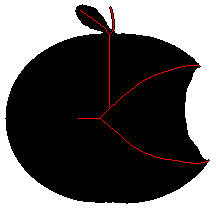} \\
		\end{tabular}
	\end{center}
	\caption{\textbf{Skeleton examples} for images apple-8 and apple-9 of the CE-Shape-1 database \cite{latecki2000}. In \textbf{(a)}, the \textbf{discrete skeleton} of the black shape is marked in red. The corresponding \textbf{distance map (b)} indicates large distances to the boundary by bright grey values. The skeleton corresponds to ridges of the distance map inside of the object.  The false colour representation of the \textbf{overlap (c)} illustrates how many skeleton points reconstruct the same object point: blue indicates the lowest amount of reconstruction overlap and red the highest. (a) and (d) together show the impact of \textbf{boundary perturbations} on the skeleton. The leaf and stem at the top of the apple create large additional branches by adding bumps to the boundary of the elliptic shape, while apple-9 introduces new cavities to the boundary that also create additional skeleton branches. \label{fig:apple}}
\end{figure}

For our shape sparsification scale-spaces, we consider binary images that arise for instance from segmentation or object detection algorithms. They separate the image into the background represented by value $0$ and the object marked by value $1$. In the continuous setting with an image domain $\Omega \subset \R^2$, we write images as functions $f: \Omega \rightarrow \{0,1\}$. We denote the set of object points by $O=\{ \bm x \in \Omega \, \vertL \, f(\bm x) = 1\}$. Thus, the binary image $f$ can be interpreted as the indicator function of the object set $O$, i.e. $f=\chi_O$. 

In particular, we assume that $O$ is bounded and closed. Thus the set is compact and the boundary $\partial O$ is included in $O$. While we can describe the object either by the set $O$ or its boundary $\partial O$, we require an alternative shape descriptor for our scale-space framework. 

\newpage
The medial axis transform (MAT) of \citet{Bl67} is the foundation on which we build our shape sparsification scale-spaces. In order to define it, we require a distance $d(\bm x, \bm y)$ for $\bm x, \bm y \in \Omega$, which we assume to be Euclidean. The MAT is a symmetric shape descriptor with respect to the boundary. Therefore, we first define the distance map, which maps each spatial location in the image domain to its distance to the object boundary $\partial O$. 

\begin{definition}[Distance Map $D$]\label{def:dm}
	The \textit{distance map} denotes the minimum distance to the boundary $\partial O$ according to
	\begin{equation}
		D: O \rightarrow \RR, \quad \bm x \mapsto \min\limits_{\bm y \in \partial O} d(\bm x,\bm y)
	\end{equation}
\end{definition}

The distance map values for each location $\bm x \in \Omega$ coincide with the radii of the largest disks that fit into the object if centred at $\bm x$.

\begin{definition}[Inscribed Disk $B_{D(\bm x)}$]\label{def:maxdisk}
	We denote the largest disk with centre $\bm x$ that is fully contained in $O$ by
	\begin{equation}
		B_{D(\bm x)}(\bm x):=\{\bm y \in \Omega \,  \vertL \, d(\bm x, \bm y) \leq D(\bm x) \}
	\end{equation}
\end{definition}

These disks are useful for defining the medial axis $\SK \subset O$. It contains all points that are centred in the object $O$, which means that they have equal distance from at least two boundary points. Thus, all medial axis points $\bm x$ are centres of inscribed disks that touch the image boundary at least at two points. This corresponds exactly to the subset of largest inscribed disks which are not contained in any other inscribed disk with a different centre $\bm y$. This allows us to define the medial axis or skeleton.

\begin{definition}[Skeleton $\SK$]\label{def:skel}
	The \textit{skeleton} or \textit{medial axis (MA)} $\SK$ is defined as the 
	set of centres of inscribed disks of maximal radius \cite{pfaltz1967}:
	\begin{equation*}   
		\Sigma := \{ \bm x \in O \, \vertL \, \forall \bm y \in O \setminus \{x\}: B_{D(\bm x)}(\bm x) \not\subset B_{D(\bm y)}(\bm y) \} \, .
	\end{equation*}
\end{definition}

In case the object $O$ is not connected, each component of $O$ will have its own separate skeleton. Note that the skeleton definition also does not require smoothness assumptions on the boundary of $O$.
In our scale-space framework, we will not only consider a single static object and corresponding skeleton, but whole families of these pairs. Therefore we explicitly introduce notations for mapping images to skeletons and vice versa.

\begin{definition}[Skeleton, Reconstruction, and Object Transform]\label{def:transforms}\phantom{}\\[-1.5em]
	\begin{itemize}
		\item The \textit{skeleton transform} $\ST(f)=\Sigma$ maps the image $f$ to its skeleton $\SK$ as defined in Def.~\ref{def:skel}.
		\item The \textit{object transform} $\RO (\SK)$ maps a skeleton back to the corresponding object according to
		\begin{equation}
			\RO (\SK) := \bigcup_{\bm x \in \SK}  B_{D(\bm x)}(\bm x) \, .
		\end{equation}
		\item The \textit{reconstruction transform} $\RC(\SK)$ maps the reconstructed skeleton back to a binary image $\RC(\SK): \Omega \rightarrow \{0,1\}$ with
		\begin{equation}
			\RC (\SK)(\bm x) := \begin{cases}
				1  &  \text{for } \bm x \in \RO (\SK) \\
				0  &  \text{else}
			\end{cases}\, .
		\end{equation}
	\end{itemize}

\end{definition}

According to the classical definitions of \citet{Bl67} and \citet{pfaltz1967}, all three transforms are well-posed. However, in the discrete setting, uniqueness cannot always be guaranteed. We discuss these aspects in the following sections~\ref{sec:skelprop} and \ref{sec:skeldisc} and consider the respective impact on our scale-space frameworks in Section~\ref{sec:skelscale}. 

Note that the transforms also imply that we can interpret operations on the skeleton in terms
	of their effect on the object and vice versa.
	Removing or inserting skeleton points modifies the set $\Sigma$ and
	therefore changes the associated object $O=\mathcal{O}(\Sigma)$.
	Conversely, modifying the object and recomputing its skeleton via the
	skeleton transform $\mathcal{S}$ leads to a modified skeleton structure.
	Thus both representations are coupled through the transforms
	$\mathcal{S}$ and $\mathcal{O}$. Operations performed in one domain
	induce corresponding changes in the other.

Typically, a single skeleton point $\bm x$ carries more information than an object point since it represents 
the whole section of the shape covered by the corresponding disk $B_{D(\bm x)}$. However, its influence on the overall shape on removal also depends heavily on the rest of the skeleton. We address in more detail with the unique impact criterion in Definition~\ref{def:impact}. 
Conversely, a very small impact on the overall shape, e.g. removing single pixels at the boundary, can have a significant impact on the skeleton as illustrated in Fig.~\ref{fig:apple}(a) and (d).

\subsection{Central Skeleton Properties}
\label{sec:skelprop}

The following four central properties of skeletons are all relevant for our scale-space considerations. All of them can be already found directly or indirectly in the original publication of \citet{Bl67}, even though formal proofs were provided later in some cases.

\medskip
\noindent
{\bf Property S1: (Equivalent Shape Descriptor)\phantomsection\label{prop:S1}}\\[1mm]
In the continuous setting, the skeleton $\SK$, together with the radii of the maximal inscribed disks provides an equivalent representation of the object, that is $\RO (\SK) = O$. Consequentially, in this setting, skeleton and reconstruction transforms are truly inverse transforms with $\RC(\ST(f))) = f$ \cite{pfaltz1967}. Furthermore, $\RO(X)$ defines a continuous mapping of any set $X$. This implies that  if $X$ is compact, then the reconstruction $\RO(X)$ is compact as well. This is relevant when we build multiple reconstructions on sparsified skeletons $X \subset \SK$.

%\newpage
\medskip
\noindent
{\bf Property S2: (Thin Set)\phantomsection\label{prop:S2}}\\[1mm]
The medial axis of any closed set consists of thin curves. Mathematically speaking, it has a Hausdorff measure of zero \cite{erdos1946hausdorff,hajlasz2020old}: $\mathcal{H}(\SK)=0$. 

\medskip
\noindent
{\bf Property S3: (Homotopy)\phantomsection\label{prop:S3}}\\[1mm]
The skeleton transform preserves the connectivity of the original shape \cite{lieutier2003any,GK03}. 

\medskip
\noindent
{\bf Property S4: (Equivariance under Euclidean Motion)\phantomsection\label{prop:S4}}\\[1mm]
Consider Euclidean transformations $\mathcal{T}$ that consist of translation, rotation, reflection, or uniform scaling. Then, the skeleton transform $\ST$ commutes with the Euclidean transformation: $\ST(\mathcal{T}(f)) = \mathcal{T}(\ST(f))$.

\subsection{Discrete Setting}
\label{sec:skeldisc}

In Section~\ref{sec:skelscale} we discuss not only continuous scale-spaces, but also their discrete or semi-discrete variants. In a \emph{spatially discrete setting}, we transition from the continuous domain $\Omega$ of image coordinates to an image with a finite resolution  $n_x \times n_y$. In this case, we use a vectorial notation $\bm f \in \{0,1\}^{n}$ where the indices  $\Omega_D=\{1,...,n\}$ correspond to a row-by-row sequence of the $n=n_x \cdot n_y$ pixels. We transfer Definitions \ref{def:dm} and \ref{def:maxdisk} to the discrete setting by replacing $\Omega$ by $\Omega_D$ and using the Euclidean distance between pixel centres. Then, the skeleton from Definition~\ref{def:skel} and the transforms from Definition~\ref{def:transforms} carry over as well. However, not all of the properties are preserved. 

On a discrete pixel grid, neither the exact positions of the medial axis points nor the shape of disks can be reproduced exactly. Thus, enforcing a thin skeleton in the sense of property \hyperref[prop:S2]{S2} with one pixel width might also imply that not all points of the original shape are covered by discrete disks. Therefore, reconstructions according to Definition~\ref{def:transforms} do not necessarily reconstruct the full object, i.e. the case  $\RO(\SK) \neq O$ and $R(S(f)) \neq f$ can occur.

The equivariance Property~ \hyperref[prop:S2]{S4} still holds for on-grid translations, reflections, and rotations in $90^\circ$ degree steps. However, rescaling operations and rotations by other angles affect the discrete boundary of the object. For these cases there is no guarantee that computing the skeleton from the transformed shape is identical to transforming the discrete skeleton. This may even vary depending on the algorithm used to compute the discrete skeleton. The choice of this algorithm also affects Property~\hyperref[prop:S2]{S4}. In our work, we always use homotopy preserving skeleton computations, namely the maximal disk thinning algorithm \cite{PB12}. For such homotopy preservation, the following definitions are vital:

\begin{definition}[Endpoints, Branching Points, Simple Points, and Arcs]\label{def:points}
Let a 4-neighbourhood contain left, right, upper, and lower adjacent pixels, while an 8-neighbourhood additionally contains the diagonally adjacent pixels.	
	\begin{itemize}
		\small
		\item An \textbf{endpoint} $\bm x$ of $\SK$ has either: 1.) zero or one skeleton points in their 8-neigh\-bour\-hood; 2.) two skeleton points in its 8-neighbourhood which are adjacent to one another; 3.) three skeleton points in its 8-neighbourhood that are either all above, below, left, or right of $\bm x$. We denote the set of endpoints by $E(\SK)$.
		\item A \textbf{branching point} of $\SK$ has at least three skeleton points in its 8-neigh\-bour\-hood which are not horizontally or vertically adjacent. We denote the set of branching points by $B(\SK)$.
		\item A \textbf{simple point} is a skeleton point which is neither an endpoint nor a branching point. We write $S(\SK) := \SK \setminus (E(\SK) \cup B(\SK))$.
		\item We define $\mathcal{A}(\SK)$ as the set of all arcs/branches. Each of its elements $A = \{a_1,...,a_k\} \subset \SK$ fulfils two properties: 1.) The arc is \textit{connected}, i.e. for each $i$ the points $a_i$ and $a_{i+1}$ are 8-neighbours. 2.) Exactly $a_1$ and $a_k$ are \textit{end}- or \textit{branching points}, the rest of the $a_i$ are simple. 
	\end{itemize}
\end{definition}

In some cases the skeleton may consist solely of closed loops without
endpoints or branching points, for instance for shapes with holes such
as an annulus. In this case we have $E(\Sigma)=\emptyset$ and
$B(\Sigma)=\emptyset$, so Definition~\ref{def:points} does not yield any
arcs. We treat such special cases appropriately in our framework (e.g. in Algorithm \ref{alg:skelshape}).

These definitions are also relevant for practical applications such as shape matching. Moreover, we require them for our task-specific scale-spaces and quality measures in Section~\ref{sec:applications}.

\section{Skeletonisation Scale-Spaces}
\label{sec:skelscale}

In the following, we transfer the concept of sparsification scale-spaces  to the medial axis and propose a new class of scale-spaces. Intentionally, we keep our new class of skeletonisation scale-spaces very general. In particular, we keep the amount of requirements for our frameworks very low. As a drawback of this approach, in this section we are only able to provide generic simplification guarantees. This is well worth the cost however, since we also enjoy the liberty of designing sparsification or densification paths with very small amounts of restrictions. In Section~\ref{sec:applications}, we make use of these opportunities with a diverse set of application-specific scale-spaces that illustrate how we can equip our foundational frameworks with additional requirements that yield meaningful simplification guarantees.

For sparsification scale-spaces, \citet{CPW19} have so far only provided a fully discrete theory where both the spatial domain of image pixels and the scale are discrete. In contrast, we cover the full spectrum of continuous, discrete, and semi-discrete theory. We discuss the fully continuous and discrete setting in detail. To avoid redundancies, we treat space-continuous and scale-continuous semi-discrete scale-spaces jointly in Section~\ref{sec:semi}.

\subsection{Fully Continuous Skeleton Sparsification Scale-Spaces}

We construct a scale-space by sequentially removing skeleton points from the initial skeleton $\SK(0):=\ST(f)$ obtained from the binary image $f: \Omega \rightarrow \{0,1\}$ containing the discrete object $O=\{ \bm x \in \Omega \, \vertL \, f(\bm x) = 1\}$. With the artificial time parameter $t>0$ we represent the evolution of the skeleton $\SK(t)$ starting with initial time $t=0$. For increasing $t$, we transition to coarser scales.

\begin{definition}[Skeleton Evolution]\label{def:cont_evo}
	Given the initial skeleton $\SK(0)$, a \textit{continuous sparsification path} $\SK: [0,\infty) \rightarrow \mathcal{P}(\SK(0))$ is a function that maps any positive time $t$ to an element of the power set $\mathcal{P}(\SK(0))$ of the initial skeleton. This powerset contains all possible subsets of the skeleton that can result from subsequent removal of skeleton points.
	The evolution needs to fulfil the following requirements:
		\begin{enumerate} 
		\small
		\item \textbf{CR1 (nested skeletons):\phantomsection\label{prop:CR1}} The skeletons $\SK(t)$ are nested over time, i.e. $\SK(t_2) \subset \SK(t_1)$ for $t_2 > t_1$.
		\item \textbf{CR2 (enveloping disk):\phantomsection\label{prop:CR2}} There exists an enveloping closed disk $D(t)$ with linearly decreasing radius $r(t):= \max\{-t \cdot s + r_0, 0\}$  with $r_0 > 0$ and $s > 0$, and centre $c_D \notin \SK(0)$ such that at any time $t>0$, the skeleton $\SK(t)$ is included in the enveloping disk: $\SK(t) \subset D(t)$. 
		\end{enumerate}
\end{definition}

Since we want our scale-space framework to cover a broad range of relevant 
evolutions, we keep the requirements very general. The nested skeleton criterion \hyperref[prop:CR1]{CR1} 
ensure that once a skeleton point has been removed, it is never added back. The 
enveloping disk criterion \hyperref[prop:CR2]{CR2} is less intuitive, but we discuss why we took 
this specific design choice in Section~\ref{sec:envelope}. \hyperref[prop:CR2]{CR2} is an auxiliary construction designed to guarantee 
a minimal shrinkage speed $s$ and makes the discrete and continuous setting 
more consistent. As such, we can choose the disk centre freely outside of the skeleton. Based on the skeleton evolution, we define the sparsification 
path that indicates, which parts of the skeleton are removed.

\begin{definition}[Continuous Sparsification Path]\label{def:cont_sparspath}	
	Given a skeleton $\SK$, a \textit{continuous sparsification path} $P: [0,\infty) \times  [0,\infty) \rightarrow \mathcal{P}(\SK)$ is a function that maps a time frame with start $t_1$ and end $t_2>t_1$ to the subset $P(t_1,t_2)$ of the skeleton $\SK$ that is removed in this time frame: 
	\[
		P(t_1, t_2) := \SK(t_1) \setminus \SK(t_2)
	\]
\end{definition}

\noindent
In particular, for all $t \geq 0$ it holds that $P(t,t)=\emptyset$. We can now define our scale-space.

\begin{definition}[Continuous Skeleton Sparsification Scale-Space] \label{def:cont_skelspace}
	Consider the binary image $f: \Omega \rightarrow \{0,1\}$ with domain $\Omega$, skeleton $\SK = \ST(f)$, reconstruction $\RC(\SK)$ and sparsification path $P(\cdot, \cdot)$.
	The \textit{continuous skeleton sparsification scale-space} is the family $(u(t), \SK(t))$ of images $u(t,\bm x): [0,\infty) \times \Omega \rightarrow \{0,1\}$ and evolving skeletons $\SK: [0,\infty) \rightarrow \mathcal{P}(\SK)$ obeying
	\begin{enumerate}
		\small
		\item $\SK(0) = \SK = \ST(f)$,
		\item $\SK(t) = \SK(0) \setminus P(0,t)$
		\hspace{30mm} \textnormal{for} $t > 0$,
		\item $u(t) := \RC(\SK(t))$ \textnormal{and} $O(t):=\RO(\SK(t))$ \hspace{7mm}  \textnormal{for} $t \geq 0$.
	\end{enumerate}
\end{definition}

\subsubsection{Continuous Scale-Space Properties}

Most of the properties \hyperref[prop:C1]{C1}-\hyperref[prop:C6]{C6} discussed in the following verify classical scale-space concepts of \citet{AGLM93}. The presence of corresponding architectural, simplification, and equivariance features identifies our newly constructed framework as a scale-space. Only Property~\hyperref[prop:C4]{C4} is entirely specific to our shape analysis case since it verifies that sparsification does not destroy important skeleton characteristics. All other properties have classical scale-space analogues.

\medskip
\noindent
{\bf Property C1: (Original Skeleton and Image as Initial State)\phantomsection\label{prop:C1}}\\[1mm]
By Definition~\ref{def:cont_skelspace}, we have $\SK(0) = \ST(f) = \SK$. Due to Property~\hyperref[prop:S1]{S1} of the medial axis transform, $u(0)=\RC(\SK(0))=\RC(\ST(f))=f$. 

\medskip
\noindent
{\bf Property C2: (Causality)\phantomsection\label{prop:C2}}\\[1mm]
We can equivalently reach the skeleton-image-pair $(u(t),\SK(t))$ in a step of size $t$ from time $0$ or in a step of size $t-\hat{t}$ from time $\hat{t}<t$. This holds according to Definitions \ref{def:cont_sparspath} and \ref{def:cont_skelspace} since
\begin{equation}
	\SK(0) \setminus P(0,t) = \SK(t) = \SK(\hat{t}) \setminus P(\hat{t},t) \, .
\end{equation}
In fact, due to the nested set property \hyperref[prop:CR1]{CR1}, we always have $P(t_2,t_3) \subset P(t_1,t_3)$ for $t_1 < t_2 < t_3$.

\medskip
\noindent
{\bf Property C3: (Lyapunov Sequences)\phantomsection\label{prop:C3}}\\[1mm] 
Scale-spaces need quantifiable simplification over time. Due to the generality of our framework, we only prove generic Lyapunov sequences here. For concrete applications in Section~\ref{sec:applications}, we can find Lyapunov sequences that reflect task-specific measures of simplification. A simple Lyapunov sequence can be defined via the object area.

\begin{proposition}[Decreasing Object Area] 
	\label{prop:area}
	\noindent The object area $a(t) := \mathcal{H}(O(t))$ in terms of its Hausdorff measure decreases over time $t$, i.e. $a(t_2) \leq a(t_1)$ for $t_2 > t_1$.
\end{proposition}

\begin{proof}
	For a time $t_2 > t_1$, we have $(\ast): \SK(t_2) \subset \SK(t_1)$ due to requirement \hyperref[prop:CR1]{CR1}. Therefore, we can infer
	\begin{equation}
			\RO(\SK(t_2))  = \bigcup_{\bm x \in \SK(t_2)} B_{D(\bm x)}(\bm x) \stackrel{(\ast)}{\subset}  \bigcup_{\bm x \in \SK(t_1)} B_{D(\bm x)}(\bm x) = \RO(\SK(t_1))
	\end{equation}
	From $\RO(\SK(t_2)) \subset  \RO(\SK(t_1))$ we can conclude $a(t_2) \leq a(t_1)$. 
\end{proof}

Instead of the area, we can also use the diameter $\dia(\RO(\SK(t)))$ of the reconstructed object to define a Lyapunov sequence. For a closed set $S \subset \Omega$ we define 
\begin{equation*}
	\dia(S) := \max\{ d(\bm x, \bm y) \vertL \bm x, \bm y \in S \},
\end{equation*}
where $d(i,j)$ is the Euclidean distance. 

\begin{proposition}[Decreasing Object Diameter] 
	The object diameter decreases with increasing time $t$, i.e. for $t_2 > t_1 \geq 0$, we have
	\begin{equation}
		\dia(\RO(\SK(t_2))) \leq \dia(\RO(\SK(t_1))).
	\end{equation}
\end{proposition}

\begin{proof}
	Assume that $\dia(\RO(\SK(t_2))) > \dia(\RO(\SK(t_1)))$.  Then, there are $\bm x, \bm y \in \RO(\SK(t_2))$ with $d(\bm x, \bm y) > \dia(\RO(\SK(t_1)))$. However, we know that for $t_2 > t_1$, we have $\RO(\SK(t_2)) \subset  \RO(\SK(t_1))$ according to our previous proof for Proposition~\ref{prop:area}.
	Therefore, we also have $\bm x, \bm y \in \RO(\SK(t_1))$ which contradicts 
	our assumption. Thus, the object diameter at $t_2$ cannot be larger than 
	the diameter at time $t_1$.
\end{proof}

\medskip
\noindent
{\bf Property C4: (Skeleton Properties)\phantomsection\label{prop:C4}}\\[1mm] 
The properties \hyperref[prop:S1]{S1} (equivalent shape descriptor) and \hyperref[prop:S2]{S2} (thin set) remain fulfilled under sparsification. We cover the equivariances \hyperref[prop:S2]{S4} separately.
\begin{description}
	\small
	\item[S1:] By definition, $u(t) = \RC(\SK(t))$.
	\item[S2:] For all $t>0$, we have $\SK(t) \subset \SK$ and $\mathcal{H}(\SK)=0$. Therefore, also $\mathcal{H}(\SK(t))=0$ and thus all evolving skeletons are thin sets.
\end{description}
Note that we intentionally do not require the homotopy preservation Property~\hyperref[prop:S2]{S3} to be fulfilled for all of our scale-spaces. This means that due to sparsification, a shape can potentially get disconnected. In Section~\ref{sec:shaperec}, we show how additional requirements can add this property if desired.

\medskip
\noindent
{\bf Property C5: (Equivariance)\phantomsection\label{prop:C5}}\\[1mm] 
In Section~\ref{sec:skelprop} we have reviewed the inherent equivariance property \hyperref[prop:S2]{S4} of the medial axis: The MAT commutes with any Euclidean transformation $\mathcal{T}$, including translation, rotation, reflection, and scaling. In Property~\hyperref[prop:C4]{C4}, we have also established that the reconstruction criterion $u(t) = \RC(\SK(t))$
holds for all $t$. Thus, for all $t>0$, we have skeleton/image pairs for which the equivariance $\ST(\mathcal{T}(u(t))) = \mathcal{T}(\ST(u(t))) = \mathcal{T}(\SK(t))$ holds. Consequentially, the equivariances transfer directly from the MAT to the whole scale-space evolution: If we apply the transform $\mathcal{T}$ to $f$ at time $t=0$, the image skeleton pairs are transformed accordingly by $\mathcal{T}$. 

\medskip
\noindent
{\bf Property C6: (Empty Image as Steady State)\phantomsection\label{prop:C6}}\\[1mm]
We know that the enveloping disk $D$ is centred in $\bm c_D \notin \SK$. Let $r_m$ denote the minimal distance of $\bm c_D$ to the set $\SK$. At time $T > \frac{r_0 - r_m}{s}$, we thus get $D(T) \cap \SK(0) = \emptyset$. Since $\SK(T) \subset \SK(0)$, the requirement $\SK(T) \subset D(T)$ can only be fulfilled for $\SK(T) = \emptyset$. Thus, the reconstruction is an empty image as well with object $O(t) = \emptyset$ and we reach an empty steady state in finite time. 

\subsubsection{Impact of the Enveloping Disk Requirement}
\label{sec:envelope}

If we consider the proofs of our six scale-space properties, the enveloping disk requirement of Definition~\ref{def:cont_sparspath} has only an impact on Property 6, the steady state. It ensures that the skeleton is removed completely in finite time and thus the final reconstruction is an empty image. Alternatively, we could require the nesting property \hyperref[prop:CR1]{CR1} to enforce proper subsets $\SK(t_2) \subsetneq \SK(t_1)$ for $t_2 > t_1$.

 However, the intersection of all $\SK(t)$ would still be non-empty and (given a shrinking diameter) would converge to a single skeleton point.
While this would constitute a valid steady state evolution in line with typical scale-spaces, it would deviate from the discrete case which has an empty steady state as we discuss in Section~\ref{sec:fulldisc}. Here, for the sake of consistency, we have chosen to enforce empty steady states for both. Alternatively, we could also modify the requirements for discrete scale-spaces such that their steady state becomes a single skeleton point.

Note that the linear shrinkage of the enveloping disk is not an overly restrictive requirement for possible sparsification paths. It is an auxiliary construction that only influences the implicit parametrisation of the artificial time variable. Nonlinear and asymmetric shrinkage of the skeleton is not ruled out by the definition, as long as this shrinkage is overall fast enough relative to the reference speed $s$ provided by the shrinkage of the enveloping disk. Thus, with an appropriate choice of $s$ and $c_D$, all skeleton sparsification paths of interest are possible.

\subsection{Fully Discrete Skeleton Sparsification Scale-Spaces}
\label{sec:fulldisc}

The fully discrete setting is relevant for our practical implementations in Section~\ref{sec:applications}. Here we consider an $n_x \times n_y$ discrete binary image $\bm f \in \{0,1\}^{n_x n_y}$ with discrete image domain $\Omega_D = \{1,...,n_x n_y\}$. In contrast to the fully continuous setting, it is more convenient to explicitly define the sparsification path first, especially since we are going to use it as a design criterion for the scale-space later on.

\begin{definition}[Discrete Sparsification Path]\label{def:sparspath}	
	Given a skeleton $\SK \subset \Omega_D$, a \textit{sparsification path} $P = (P_1,\dots,P_{m})$ with $m \in \NN^+$ is an ordered collection of non-empty sets that form a partition of $\SK$, i.e.: 
	\begin{enumerate}
		\item $P_k \neq \emptyset$,
		\item $P_k \cap P_\ell = \emptyset$, $k \neq \ell$
		\item $\bigcup_{\ell=1}^{m}P_\ell = \SK$. 
	\end{enumerate}
\end{definition}

Based on this finite number of sparsification steps, we can define the scale-space. It is not only discrete in space, but also has a discrete scale parameter $\ell$.

%\newpage

\begin{definition}[Discrete Skeleton Sparsification Scale-Space] \label{def:skelspace}
	Consider the discrete binary image $\bm f \in \{0,1\}^{n_x n_y}$ with domain $\Omega_D$, skeleton $\SK = \ST(f)$, reconstruction $\RC(\SK)$ and sparsification path $P=(P_1,...,P_m)$ partitioning $\SK$, $m \in \NN^+$.
	The \textit{discrete skeletonisation scale-space} is the family $(\bm u_\ell, \SK_\ell)_{\ell=0}^{m}$ of images $\bm u_\ell$ and skeletons $\SK_\ell$ obeying
	\begin{enumerate}
		\small
		\item $\SK_0 := \SK = \ST(\bm f)$,
		\item $\SK_\ell := \SK_0 \setminus \bigcup_{i=1}^{\ell} P_i$
		\hspace{22mm} \textnormal{for} $\ell \in \{1,\dots,m\}$,
		\item $\bm u_\ell := \RC(\SK_\ell)$ \textnormal{and} $O_\ell =\RO(\SK_\ell)$ \hspace{7mm}  \textnormal{for} $\ell \in \{0,\dots,m\}$.
	\end{enumerate}
\end{definition}

As in the continuous case, the sparsification path defines which points to prune in the transition from scale $\ell$ to scale $\ell+1$. However, the number $m$ of path sets uniquely determines the finite number $m+1$ of discrete scales. Both the number of steps and the composition of the sets can be freely chosen as long as the partition condition is fulfilled. Note that we also require $P_\ell \neq \emptyset$ and therefore at least one skeleton point is removed per sparsification steps. This implies a maximum of $m \leq |\SK|$ sparsification steps.

\subsection{Discrete Scale-Space Properties}
\label{sec:properties}

The following properties \hyperref[prop:D1]{D1}-\hyperref[prop:D6]{D6} are discrete counterparts to \hyperref[prop:C1]{C1}-\hyperref[prop:C6]{C6} from the continuous setting. Many of the properties carry over in very similar fashion, but there are a few notable exceptions.

\medskip
\noindent
{\bf Property D1: (Original Skeleton as Initial State)\phantomsection\label{prop:D1}}\\[1mm]
As in the continuous setting, $\SK_0 = \ST(\bm f)$ at scale $\ell=0$ is the 
skeleton of the original object due to Definition~\ref{def:skelspace}. However, 
as we have pointed out in Section~\ref{sec:skeldisc}, there is no guarantee 
that $\RC(\SK)=\bm f$ in the discrete setting. Thus, the initial image $\bm 
u_0$ is typically only an approximation to $\bm f$.

\medskip
\noindent
{\bf Property D2: (Causality)\phantomsection\label{prop:D2}}\\[1mm]
By following the sparsification path $P = (P_1,\dots,P_m)$, a scale $\ell \in \{0,...,m\}$ can be reached either from the initial scale $0$ or from from any intermediate scale $k \in \{0,...,\ell-1\}$. Similarly to the continuous case, we use Definition~\ref{def:skelspace} to verify
\begin{equation}
	\SK_\ell = \SK_0 \setminus \bigcup_{i=1}^{\ell}P_i = \left(\SK_0 \setminus \bigcup_{i=1}^{k}P_i \right) \setminus \bigcup_{i=k+1}^{\ell}P_i =  \SK_{k} \setminus \bigcup_{i=k+1}^{\ell}P_i \, .
\end{equation}
In the discrete setting, the time spans are replaced by a concrete number of steps: Scale $\ell$ can be reached in $\ell$ steps from the initial scale or in $\ell-k$ steps from scale $k$. 

\medskip
\noindent
{\bf Property D3: (Lyapunov Sequences)\phantomsection\label{prop:D3}}\\[1mm] 
The same two generic Lyapunov sequences as in the continuous setting apply. The proof for the \emph{decreasing object diameter} remains unchanged and therefore, we do not repeat it here. For the \emph{decreasing object area}, we have slight adaptations due to deviations in the definitions, but the proof remains very similar. Here, the image area is simply defined by the cardinality $\abs{\cdot}$ of the object set.

\begin{proposition}[Decreasing Object Area] 
	\label{prop:areadisc}
	\noindent The area $a_\ell := \abs{O_\ell}$ of the object decreases as the scale parameter $\ell$ increases, i.e. $a_k \leq a_\ell$ for $k > \ell$.
\end{proposition}

\begin{proof}
	For the path $P = (P_1,\dots,P_m)$ we have $\SK_\ell = \SK_{\ell+1} \cup P_{\ell+1}$ and thus
	\begin{equation*}
		\quad \bigcup_{\bm x \in \SK_\ell} B_{D(\bm x)}(\bm x)\quad = \quad \bigcup_{\bm x \in \SK_{\ell+1}} B_{D(\bm x)}(\bm x) \:\cup\: \bigcup_{\bm x \in P_{\ell+1}} B_{D(\bm x)}(\bm x) \, .
	\end{equation*}
	Therefore, we can derive
	\begin{equation*}
		\RO(\SK_\ell) = \quad \bigcup_{\bm x \in \SK_\ell} B_{D(\bm x)}(\bm x) \:\supseteq \bigcup_{\bm x \in \SK_{\ell+1}} B_{D(\bm x)}(\bm x) \quad = \RO(\SK_{\ell+1}) \, .
	\end{equation*}
	Thus, we can conclude $a_\ell = \abs{ \RO(\SK_\ell) }  \geq  \abs{ \RO(\SK_{\ell+1}) } = a_{\ell +1}$. 
\end{proof}

In general the object area is not strictly decreasing. Although	each sparsification step removes at least one skeleton point
	($P_{\ell+1}\neq\emptyset$), the reconstructed object may remain unchanged if all disks associated with the removed points are fully covered by disks
	of the remaining skeleton.

\medskip
\noindent
{\bf Property D4: (Skeleton Properties)\phantomsection\label{prop:D4}}%\\[1mm] 
%\begin{description}

	The property follows from the corresponding properties of the discrete
	skeletonisation and reconstruction operators.
	
	Even though discrete skeletons do not necessarily fulfil the reconstruction criterion \hyperref[prop:S1]{S1}, $u_\ell = \RC(\SK_\ell)$ is fulfilled by Definition~\ref{def:skelspace}. Since the object is created from the sparsified skeleton itself, there are no discretisation artefacts beyond the potential mismatch to the original shape as discussed in Property~\hyperref[prop:D1]{D1}.
	
	By definition, the initial skeleton $\SK$ has one pixel thickness. Removing pixels cannot increase this thickness and therefore, the property \hyperref[prop:S2]{S2} is preserved for all $\Sigma_\ell$, $\ell \in \{0,...,m\}$.
	
	Together, these properties ensure that the skeletonisation and reconstruction operators are
	consistent with each other in the discrete setting, analogous to the continuous case discussed in Property~\hyperref[prop:C4]{C4}.

As in the continuous setting, removing skeleton points without additional constraints can disconnect the skeleton graph and thus change the topology of the skeleton. The homotopy preservation according to \hyperref[prop:S2]{S3} is thus not generally fulfilled.

\medskip
\noindent
{\bf Property D5: (Equivariance)\phantomsection\label{prop:D5}}\\[1mm] 
As in the continuous setting, the equivariance properties directly carry over from the medial axis to the scale-space evolution. However, following our discussion of discrete skeletons in Section~\ref{sec:skeldisc}, this does not apply to all Euclidean transformations, only to on-grid translations, $90^\circ$ rotations, and mirroring along image axes.

\medskip
\noindent
{\bf Property D6: (Empty Image as Steady State)\phantomsection\label{prop:D6}}\\[1mm]
Since $P_1,\dots,P_m$ partition $\SK$ and $\SK_0 = \SK$ by Definition~\ref{def:sparspath} and \ref{def:skelspace}, we can immediately conclude $\SK_m =  \SK \setminus \bigcup_{\ell=1}^m P_\ell = \SK \setminus \SK = \emptyset$. Since the final skeleton in empty, the steady state is an empty image.

\medskip

In summary, the discrete scale-spaces capture all major properties of the continuous setting. Only discretisation artefacts in the reconstruction and Euclidean transformations limit properties \hyperref[prop:D1]{D1} and \hyperref[prop:D4]{D4} slightly compared to their continuous counterparts. This shows that our discrete implementations capture the main ideas of 	the proposed scale-space framework. While several properties carry over
	directly to the discrete setting, others only hold in a weaker form or
	require additional considerations due to discretisation effects.

\subsection{A Note on Semi-Discrete Skeletonisation Scale-Spaces}
\label{sec:semi}

It is also possible to construct scale-spaces that are time-discrete and space-continuous or vice versa. Since properties and proofs closely resemble a combination of \hyperref[prop:C1]{C1}-\hyperref[prop:C6]{C6} and \hyperref[prop:D1]{D1}-\hyperref[prop:D6]{D6}, we do not discuss them here in detail to avoid redundancies. 

The time-continuous and space-discrete case inherits the drawbacks of both the fully discrete and the fully continuous scale-space: Equivariances and initial state are weakened by spatial discretisation artefacts. Moreover, the time continuous setting requires an additional criterion like the enveloping shrinking disks to make it consistent with the time-discrete steady state.

As expected, the second semi-discrete scale-space enjoys the benefits of both worlds. Since it is spatially continuous, it inherits full Euclidean equivariance and a perfect reconstruction of the initial shape. Since the scale parameter is discrete, we can define the sparsification path directly as in Definition~\ref{def:sparspath}. We only need to replace the discrete image $\bm f$ and the domain $\Omega_D$ by their continuous counterparts. In our previous conference publication~\cite{GP25}, we have discussed this case jointly with the fully discrete scale-space framework.

\subsection{Skeleton Densification Scale-Spaces}

All previous scale-space frameworks were motivated by the sparsification 
paradigm that was also applied by \citet{CPW19}. This is a top-down strategy, 
which starts with the full known data and gradually reduces it until nothing is 
left. However, in the literature on data optimisation, where sparsification has 
its origins, there is also an inverse counterpart: Densification 
\cite{CW21,DAW21,KBPW18} starts with an empty image and adds new known data in 
every step. In data optimisation, both strategies co-exist since they have 
different strengths and weaknesses (see Section~\ref{sec:rel_opt}). We want our 
scale-space theory to cover as many existing practical applications as 
possible. Therefore, we define densification scale-spaces as an alternative to 
sparsification.

To avoid redundancies to previous sections, we limit ourselves to the fully discrete theory and only discuss noteworthy differences. A fully continuous or semi-discrete theory can be constructed in an analogous way.

\subsubsection{Densification Scale-Spaces}

In a straight-forward way we can define a densification scale-space by inverting the role of the sparsification path. We require a \emph{densification path} to adhere to the same requirements as in Definition~\ref{def:sparspath}: The sets $P_1,\dots,P_m$ are all non-empty and form a partition of the skeleton $\SK$. However, our initial skeleton $\SK_0 = \emptyset$ is now empty and the densification path specifies which points are \emph{added} to the skeleton to transition from scale $\ell$ to $\ell+1$. This leads us to the following scale-space definition.

\begin{definition}[Discrete Skeleton Densification Scale-Space] \label{def:densspace}
	Consider the discrete binary image $\bm f \in \{0,1\}^{n_x n_y}$ with domain $\Omega_D$, skeleton $\SK = \ST(f)$, reconstruction $\RC(\SK)$ and sparsification path $P=(P_1,...,P_m)$ partitioning $\SK$, $m \in \NN^+$.
	The \textit{discrete skeletonisation scale-space} is the family $(\bm u_\ell, \SK_\ell)_{\ell=0}^{m}$ of images $\bm u_\ell$ and skeletons $\SK_\ell$ obeying
	\begin{enumerate}
		\small
		\item $\SK_0 := \emptyset$,
		\item $\SK_\ell
		=
		\SK_0 \cup \bigcup_{i=1}^{\ell} P_{m-i}.$
		\hspace{16.7mm} \textnormal{for} $\ell \in \{1,\dots,m\}$,
		\item $\bm u_\ell := \RC(\SK_\ell)$ \textnormal{and} $O_\ell =\RO(\SK_\ell)$ \hspace{7mm}  \textnormal{for} $\ell \in \{0,\dots,m\}$.
	\end{enumerate}
\end{definition}

The causality property D2, the skeleton properties \hyperref[prop:D4]{D4} and equivariances \hyperref[prop:D5]{D5} all carry over from sparsification scale-spaces directly with the same arguments. All other properties are only slightly adjusted to the new setting.

\medskip
\noindent
{\bf Property DD1: (Empty Skeleton and Image as Initial State)\phantomsection\label{prop:DD1}}\\[1mm]
By definition, $\SK_0 = \emptyset$ at scale $\ell=0$ and thus the corresponding reconstructed image $u_0 = \RC(\emptyset)$ is empty as well.

\medskip
\noindent
{\bf Property DD3: (Generalised Lyapunov-type Sequences)\phantomsection\label{prop:DD2}}\\[1mm] 
The \emph{object diameter} and \emph{object area} are increasing. While Lyapunov sequences are typically decreasing, the reverse direction of the densification space has increasing sequences as a natural counterpart to the the densification setting. The corresponding proofs carry over from \hyperref[prop:D3]{D3} with minimal adaptations.

\medskip
\noindent
{\bf Property DD6: (Original Skeleton as Steady State)\phantomsection\label{prop:DD3}}\\[1mm]
Since $P_1,\dots,P_m$ partition $\SK$, we obtain $\SK_m := \SK_0 \cup \bigcup_{i=1}^{m} P_i = \SK$ according to Definition~\ref{def:densspace}. The corresponding image $u_m = \RC(\SK)$ approximates the original object according to the discrete skeletonisation algorithm used to compute $\SK$. 

\subsubsection{Extended Densification}
\label{sec:densext}

The previous Section shows that our theory extends to the inverse counterpart of sparsification. This result alone is not very surprising, but it also means that both directions in the scale-space are equally well-posed. In classical scale-space theory this is typically not the case. Going from fine to coarse scales is usually well-posed, but going from coarse to fine scales is often problematic due to ambiguities: For instance, a skeleton consisting of a single point could be a legitimate subset of an infinite number of larger skeletons containing said point. All of them constitute valid densifications of the initial skeleton. This noteworthy fact also opens up some entirely new perspectives. 

The steady-state in sparsification is final by design. All points have been removed and there is no interesting information left in the empty image. In contrast, densification stops when the full skeleton is reached. This is a natural stopping point, but there are still interesting points left to be added. While the skeleton might give a full representation of the original object (in the continuous setting), this is not necessarily the case for discrete skeletons. There are many practical reasons why one might want to add more points. For instance, this could be the densification counter-part to pruning: Instead of removing erroneous branches, one could correct overly strict thinning applied by the skeletonisation algorithm. It might also be interesting to consider the transition to more general descriptors such as symmetry sets, a superset of skeletons. In Section~\ref{sec:applications}, we highlight a recursive extension of the skeleton that adds additional points that yield a shape representation that is overcomplete, but relevant, e.g., for 3-D printing.

Such generalisations of the skeleton scale-spaces can be obtained by simply lifting the requirement that the densification path forms a partition of the original skeleton, extending it to the whole object instead. 

\begin{definition}[Extended Densification Path]\label{def:extpath}	
	Given an object $O \subset \Omega_D$, an \textit{extended densification path} $P = (P_1,\dots,P_{m})$ with $m \in \NN^+$ is an ordered collection of non-empty sets that form a partition of $O$, i.e.:
	\begin{enumerate}
		\item $P_k \neq \emptyset$, 
		\item $P_k \cap P_\ell = \emptyset$, $k \neq \ell$,
		\item and $\bigcup_{\ell=1}^{m}P_\ell = O$.
	\end{enumerate} 
\end{definition}

At first glance, this seems like a full departure from skeletons to a pure 
shape densification. However, the scale-space definition still uses the 
reconstruction transform from an \emph{overcomplete skeleton} $\SE$. The 
definition of such a scale-space is identical to Definition~\ref{def:densspace} 
except for the relaxed definition of the densification path. Still, we 
have $\SK_0 := \emptyset$, $\SK_\ell := \SK_0 \cup \bigcup_{i=1}^{\ell} P_i$, 
and  $\bm u_\ell := \RC(\SK_\ell)$ for  $\ell \in \{0,\dots,m\}$.

The initial state \hyperref[prop:DD1]{DD1}, causality \hyperref[prop:DD2]{DD2}, Lyapunov Sequences \hyperref[prop:DD3]{DD3}, and equivariances DD5 are all not affected by this change, which follows with the same arguments as before. 

%\medskip
\newpage
\noindent
{\bf Property DD4: (Skeleton Properties)\phantomsection\label{prop:DD4}}\\[1mm] 
	By definition, $u_\ell = \RC(\SK_\ell)$. Thus, the reconstruction property still applies. In particular, for all $\ell$, we still have $\RC(\SK_\ell) \subset O$ since $\SE \subset O$. All disks around object points with radii given by the distance map are fully contained in the object $O$.
	
	However, \hyperref[prop:S2]{S2} is not fulfilled any more since $\Gamma$ since the extended
		representation $\Gamma$ is not restricted to thin skeleton pixels. During densification additional pixels from the reconstructed object
		may be inserted, so that $\Gamma$ can locally become thicker than a
		discrete thin skeleton. This behaviour is illustrated later in
		Section~\ref{sec:stiffexp}, where the densification with extension
		progressively fills parts of the reconstructed object.

\medskip
\noindent
{\bf Property DD6: (Full Shape and Original Image as Steady State)\phantomsection\label{prop:DD6}}\\[1mm]
Since $P_1,\dots,P_m$ partition $O$, we obtain $\SE_m = O$ and thus $\bm u_m = \RC(O) = \bm f$. Both the final overcomplete skeleton $\SE$ and its reconstruction are identical to the original object. Note that for pixels on the object boundary we consider a reconstruction with disks that contain only the boundary pixel itself. 

\medskip

Surprisingly, this sweeping generalisation of the skeletonisation scale-spaces preserves most of its properties.

%%%%%%%%%%%%%%%%%%%%%%%%%%%%%%%%%%%%%%%%%%%%%%%%%%%%%%%%%%%%%%
\section{Task-Specific Skeletonisation Scale-Spaces}
\label{sec:applications}
%%%%%%%%%%%%%%%%%%%%%%%%%%%%%%%%%%%%%%%%%%%%%%%%%%%%%%%%%%%%%%

 So far, our frameworks mostly focus on strong 
architectural and equivariance properties. In practical applications, it is 
also desirable to have task-specific, quantifiable simplification. This 
guarantees that the scale-space provides a coarse-to-fine representation that 
is useful for the task to be solved.

In this section we demonstrate with proof-of-concept examples from different fields that our framework is highly adaptable. We present three classical applications: \emph{branch pruning} as a pre-processing step of shape recognition, shape \emph{compression}, and \emph{stiffness enhancement} for 3-D printing. The first two applications rely on sparsification and have been published previously in our conference publication~\cite{GP25}. We have considerably extended the evaluations for these methods. For \emph{compression}, we also add densification experiments as an alternative. With \emph{stiffness enhancement}, we demonstrate how we can push the limits of our extended densification framework to provide a scale-space theory for applications that go beyond the classical definition of the medial axis.

\subsection{Evaluation Design}

While each of our practical examples has different focus, our experiments share common implementation details, data sets, and evaluation metrics. For the sake of reproducibility, we discuss these technical details first.

\subsubsection{Skeletonisation Implementation}

We require discrete skeletons as a starting point. As discussed in Section~\ref{sec:skeldisc}, all discrete skeletons are approximations and thus the exact definition of the discrete skeleton set $\SK$ depends significantly on the method of computation. For all of our experiments, we use the same algorithm, namely maximal disk thinning \cite{PB12} based on the skeleton membership criteria of \citet{remy2005}. This method guarantees that the skeleton has a one pixel width and preserves the homotopy of the original shape. Both for the skeleton computation and the reconstruction we require distance maps, which we compute with the method of \citet{meijster2000}. For the reconstruction itself, we rely on a simple disk drawing algorithm according to Definition \ref{def:transforms}. A sample implementation is available on Zenodo \cite{GP26code}.

\subsubsection{Data Sets}
\label{sec:datasets}

\begin{figure}[t]
	\small
	\tabcolsep4pt
	\begin{center}
		\begin{tabular}{ccc}
			original & segmentation & skeleton \\
			\hline
			\includegraphics[width=0.3\textwidth]{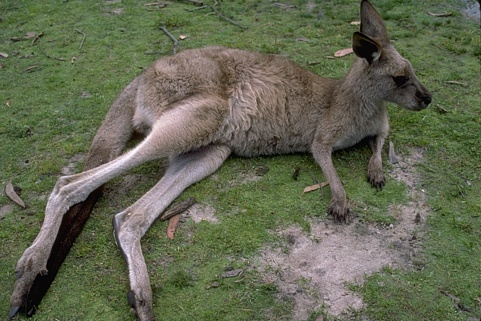} &
			\includegraphics[width=0.3\textwidth]{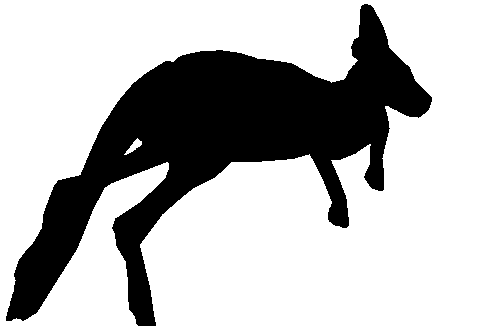} &
			\includegraphics[width=0.3\textwidth]{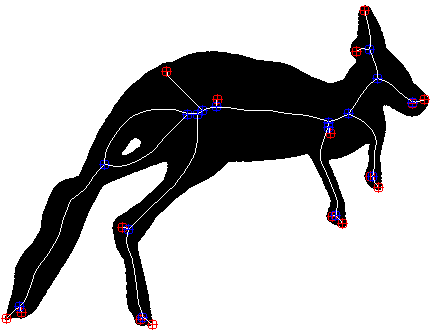} \\
		\end{tabular}
		(a) Image 69020 of BSDS500
		
		\begin{tabular}{cccc}
			bat-1 & beetle-20 & & pocket-9 \\
			\hline
			\includegraphics[height=0.18\textwidth]{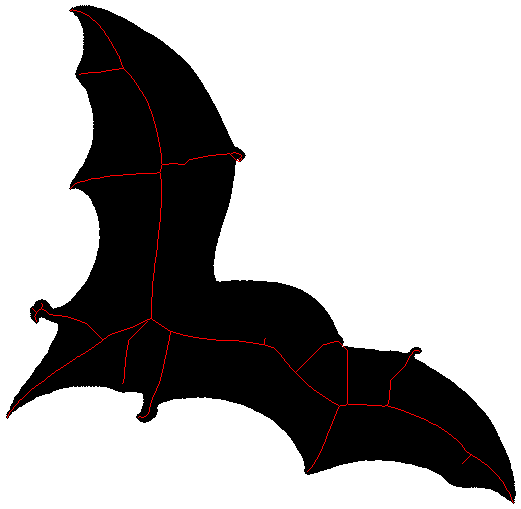} &
			\includegraphics[height=0.18\textwidth]{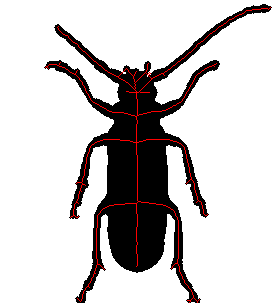} &
			\includegraphics[height=0.18\textwidth]{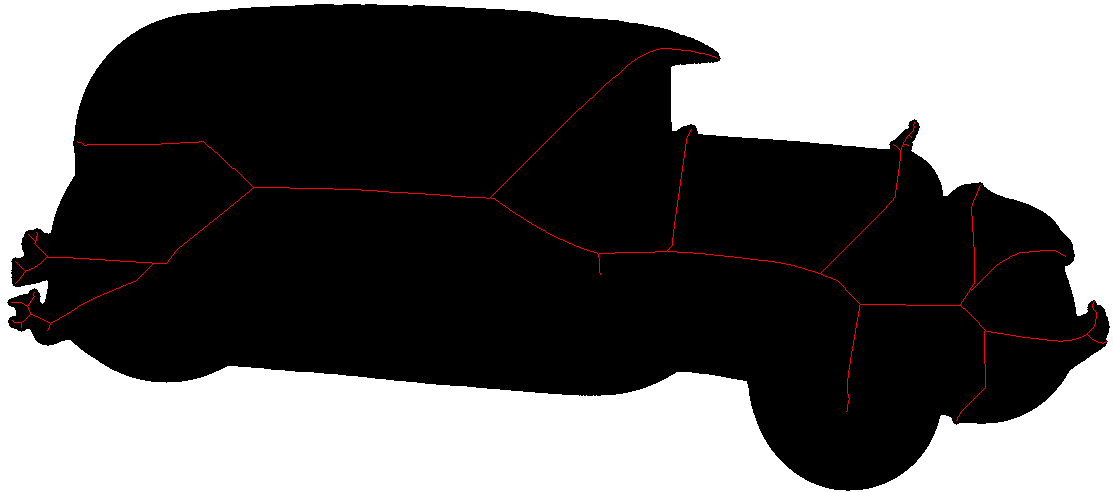} &
			\includegraphics[height=0.18\textwidth]{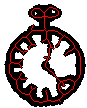}\\
		\end{tabular}
		(b) Selected Images of CE-Shape-1
	\end{center}
	\caption{We use \textbf{test images} from two databases. \textbf{(a)} With one example from the BSDS500~\cite{AMFM11} data set, we illustrate a use case where skeletonisation is part of a pipeline that extracts binary shapes from real-world images. \textbf{(b)} The 1400 images of CE-Shape-1~\cite{latecki2000} are the foundation for our quantitative evaluation. Here, we show four examples to illustrate typical shapes. They vary significantly in motive, skeleton complexity and image resolution.\label{fig:databases}}
\end{figure}

As an input to our algorithm, we require binary shapes. Therefore, we consider two test datasets: BSDS500~\cite{AMFM11} contains ground truth segmentation data for real-world images. This database acts as an example for a typical pipeline in which our scale-space framework would constitute an intermediate step. Before our algorithm is deployed, a segmentation approach extracts a binary shape from a real-world image (see Fig.~\ref{fig:databases}). Our skeleton would then typically feed into other algorithms, e.g. for shape matching. For our visual comparisons, we have extract individual objects from the ground truth segmentations for our visual comparison.
	
For our quantitative experiments, we rely on the MPEG7 core experiment CE-Shape-1 database \cite{latecki2000} instead. It contains a variety of different classes of synthetic shape images that can be directly used without the need for further pre-processing.

\subsubsection{Quality Criteria}
\label{sec:metrics}

An obvious quality measure for skeletons is the reconstruction error as the sum of false negative and false positives in a discrete reconstruction of the initial object $O$. However, there are a few more quality criteria that are useful for our task-specific scale-spaces \cite{PB12}.

\medskip
\noindent
The \textit{reconstruction error} $\mathcal{E}_\ell$ is the number of missing object points
\begin{equation}
	\mathcal{E}_\ell := \abs{O_0} - \abs{O_\ell} \, .
\end{equation}
Note that this error only considers false negatives in the object representation. False positives can not occur by definition since we only use this measure for reconstructions from subsets of the full skeleton.

\medskip

\noindent
\emph{Skeleton minimality} relates skeleton points $\abs{\SK_\ell}$ to the number $\abs{O_\ell}$ reconstructed image points. 
\begin{equation}
	\mathcal{M}_\ell :=  \frac{\abs{\SK_\ell}}{\abs{O_\ell}} \, .
\end{equation}

\medskip

\noindent
\textit{Skeleton complexity} is the total number of significant points in the skeleton:
\begin{equation}
	\mathcal{C}_\ell := \abs{E(\SK_\ell)} + \abs{B(\SK_\ell)}\, .
\end{equation}
This definition relies on the notions of endpoints $E$ and branching points $B$ (see Definition~\ref{def:points}).

\subsubsection{Baseline: Random Sparsification}

As a baseline for our task-specific scale-spaces we consider a sparsification path that at each scale $\ell$ selects a point $\bm x \in \SK_\ell$ uniformly at random. This sparsification path is completely independent of the structure of the object and does not follow any goals. Following the taxonomy of \citet{CPW19}, we call such a scale-space \emph{uncommitted}. All our application-specific scale-spaces take the shape or topology of the original object into account and are therefore \emph{committed}. 

Note that even this simple uncommitted sparsification path yields a scale-space that obeys the properties \hyperref[prop:D1]{D1}-\hyperref[prop:D6]{D6} from Section~\ref{sec:properties}.

%%%%%%%%%%%%%%%%%%%%%%%%%%%%%%%%%%%%%%%%%%%%%%%%%%%%%%%%%%%%%%
\subsection{Skeletonisation Scale-Spaces for Branch Pruning}
\label{sec:shaperec}
%%%%%%%%%%%%%%%%%%%%%%%%%%%%%%%%%%%%%%%%%%%%%%%%%%%%%%%%%%%%%%

First, we investigate classical skeleton pruning~\cite{BLL07,ML12,Og94,SB98,SS16,TH02} as a special case of our sparsification framework. Since this is often a preprocessing step for applications such as shape matching, the homotopy preservation property \hyperref[prop:S2]{S3} from Section~\ref{sec:skelprop} is important here. The connectivity of the object should be preserved. This constitutes a constraint for our sparsification path. 

In addition, end- and branching points from Definition~\ref{def:points} are semantically important for shape matching. Since endpoints and branching points mark where individual arcs of the skeleton intersect or end, they are more relevant for shape recognition than simple points. Thus, skeleton complexity as defined in Section~\ref{sec:metrics} is a good criterion for simplification. 

Pruning also aims to remove so-called spurious branches that are the result of small boundary perturbations. These can have a large effect on skeleton minimality since they can increase the size of the skeleton significantly without contributing much to the object reconstruction. We formalise this in terms of the reconstruction impact.

\begin{definition} [Reconstruction Impact]
	\label{def:impact}
	For a set of skeleton points $S \subseteq \SK_\ell$ and $\ell \in \{0,\dots,m\}$ we define the \textit{reconstruction impact} $I_{\ell,S}$ by
	\begin{equation}
		I_{\ell,S} := O_\ell \setminus \Bigl( \bigcup_{i \in \SK_\ell \setminus S} B_{D(i)}(i) \Bigr).
	\end{equation}
\end{definition}

Thus, our goals for designing a reconstruction path is removing whole branches of the skeleton without disconnecting it and avoiding increases in skeleton complexity. If we have multiple choices for branches to remove, we want to choose the ones with the least reconstruction impact first, since those are most likely to arise from boundary perturbations. We incorporate these goals into our branch pruning path described by Algorithm~\ref{alg:skelshape}. In the following, we show that the additional requirements of this pruning path lead to skeleton complexity as a new Lyapunov sequence for our task-specific scale-space.

\begin{algorithm2e}[t]
	\caption{Branch Pruning Path\label{alg:skelshape}}
	$\ell \gets 0$\;
	\While{$\abs{\SK_\ell} > 0$}{
		\tcc{Consider branches with endpoints or full skeleton.}
		$A_\mathcal{E} \gets \{ A \in \mathcal{A}(\SK_\ell)\, \vertL \, E(\SK_\ell) \cap A \neq \emptyset \}$, \lIf{$A_\mathcal{E} = \emptyset$}{$A_\mathcal{E} \gets \{\SK_\ell\}$}
		\tcc{Select branch with smallest reconstruction impact.}
		$A_{\min} \gets \textnormal{argmin}_{A \in \mathcal{A}_\mathcal{E}} \abs{I_{\ell, A}}$ \;
		\tcc{Remove whole branch excluding branching points.}
		$P_{\ell+1}  \gets \{a_i \in A_{\min} \, \vertL \, a_i \notin B(\SK_\ell) \}$,
		$\SK_{\ell+1} \gets \SK_\ell \setminus P_{\ell+1}$, 
		$\ell \gets \ell + 1$\;
	}	
\end{algorithm2e}

\begin{proposition}[Skeleton Complexity is a Lyapunov Sequence]
	Skeleton complexity decreases with increasing scale $\ell$, i.e. $\mathcal{C}_\ell \geq \mathcal{C}_{\ell+1}$.
\end{proposition}

\begin{proof} For a scale $\ell \in \{0,...,m-1\}$, there are the following possible cases for the next step $P_{\ell+1}$ in the sparsification path:
	\begin{description}
		\item[Case 1:] $P_{\ell+1}$ has exactly two endpoints, which are both removed. The branch is not connected to other skeleton components and thus cannot influence other points. Then $\mathcal{C}_{\ell+1} = \mathcal{C}_{\ell}-2 < \mathcal{C}_{\ell}$.
		\item[Case 2:] By Algorithm~\ref{alg:skelshape}, $P_{\ell+1}$ has exactly one endpoint, which is removed. The arc is connected to a branching point. Its role can change to an end or simple point, but this cannot increase complexity. Thus, $\mathcal{C}_{\ell+1} \leq \mathcal{C}_{\ell}-1 < \mathcal{C}_{\ell}$.
		\item[Case 3:] $P_{\ell+1} = \SK_\ell$. The remaining skeleton is removed, thus $\mathcal{C}_{\ell+1} = 0 \leq \mathcal{C}_{\ell}$. 
	\end{description} 
\end{proof}

%%%%%%%%%%%%%%%%%%%%%%%%%%%%%%%%%%%%%%%%%%%%%%%%%%%%%%%%%%%%%%
\subsection{Skeletonisation Scale-Spaces for Compression}
\label{sec:comp}
%%%%%%%%%%%%%%%%%%%%%%%%%%%%%%%%%%%%%%%%%%%%%%%%%%%%%%%%%%%%%%

For shape compression \cite{Mu20}, homotopy preservation is not relevant. We 
merely want to represent the initial shape as accurately as possible with as 
few skeleton points as possible. More sophisticated versions of such a 
task-specific scale-space could also incorporate coding aspects such as the 
entropy of the remaining skeleton points, but this is beyond the scope of our 
proof of concept. We refer to \citet{Mu20} for such considerations. Instead, we 
solely rely on the reconstruction impact from Definition~\ref{def:impact}.
By removing $S$ from the skeleton $\SK_\ell$, we remove $I_{\ell,S}$ from the object and thus obtain in every step
\begin{equation}
	\label{eq:removeimpact}
	O_{\ell+1} = O_\ell \setminus I_{\ell,S} \, .
\end{equation}
Since we want to minimise a loss of object points for every skeleton part removed, we have to minimise $\abs{I_{\ell,S}}$ in every step. Consequentially, our Algorithm~\ref{alg:skelcomp} defines a skeleton compression path that minimises this error. Here, we use the shorthand notation $I_{\ell,i} := I_{\ell,\{i\}}$ for a given pixel index $i \in \SK$. 
\begin{algorithm2e}[t]
	 \caption{Skeleton Compression Path\label{alg:skelcomp}}
	$\ell \gets 0$, $r \in \NN$ user parameter (points to remove per step) \; 
	\While{$\abs{\SK_\ell} > 0$}{
		$k \gets \abs{\SK_\ell}$, $s \gets \min(r,\abs{\SK_\ell})$\;
		\tcc{Order all skeleton points by reconstruction influence.}
		Let $\{c_1,...,c_k\} = \SK_\ell$ with $\abs{I_{\ell,c_i}} \leq \abs{I_{\ell,c_j}}$ for $i \leq j$\;
		\tcc{Remove $s$ points with smallest impact on reconstruction.}
		$P_{\ell+1} \gets \{c_1,...,c_s\}$, $\SK_{\ell+1} \gets \SK_\ell \setminus P_{\ell+1}$, $\ell \gets \ell + 1$\;
	}	
\end{algorithm2e}

By design, the relative error constitutes a Lyapunov sequence as a corollary to Proposition~\ref{prop:areadisc} since $\abs{O_\ell}$ is shrinking and thus the relative error is increasing with the scale. This is to be expected, but not a particularly useful property for a compression algorithm. However, we can also guarantee that the increase in error for every step is indeed minimal.

\begin{proposition}[Minimal Relative Error Increase]
	Among all possible sparsification paths, for all $\ell \in \{0,...,m-1\}$, the increase in relative error $\mathcal{E}_{\ell+1}-\mathcal{E}_\ell$ is minimal for a skeleton compression scale-space.
\end{proposition}

According to Eq.~\eqref{eq:removeimpact}, we have $\abs{O_{\ell+1}}=\abs{O_{\ell}}-\abs{I_{\ell,P_{\ell+1}}}$. Moreover, Algorithm~\ref{alg:skelcomp} selects the next step $P_{\ell+1}$ of the sparsification path such that $\abs{I_{P_{\ell+1}}}$ is minimised over all possible choices of $P_{\ell+1}$. Therefore, each step yields the minimal possible increase in reconstruction error.

Additionally, we want a small number of skeleton points to represent a large 
number of object points. This is reflected by the skeleton minimality from 
Section~\ref{sec:metrics}. It also constitutes a Lyapunov sequence for our 
compression scale-spaces.

\begin{proposition}[Skeleton Minimality is a Lyapunov Sequence]
	Skeleton minimality decreases with increasing scale $\ell$, i.e. $\mathcal{M}_\ell \geq \mathcal{M}_{\ell+1}$.
\end{proposition}

\begin{proof}
	First, we decompose $\mathcal{M}_\ell$ according to the sparsification path from Definition~\ref{def:skelspace} and the reconstruction impact from Eq.~\eqref{eq:removeimpact}. This yields
	\begin{equation}
		\label{eq:ml}
		\mathcal{M}_\ell = \frac{\abs{\SK_\ell}}{\abs{O_\ell}} = 
		\frac{\abs{\SK_{\ell+1}}+\abs{P_{\ell+1}}}{\abs{O_{\ell+1}}+\abs{I_{\ell, P_{\ell+1}}}} 
		\, .
	\end{equation}	
	Furthermore, due to the sorting in Algorithm~\ref{alg:skelcomp}, the average unique area of the removed points in $P_{\ell+1}$ is smaller or equal to the average unique area of the remaining skeleton points in $\SK_{\ell+1}$, i.e.
	\begin{equation}
		\label{eq:avgarea}
		\frac{\abs{I_{\ell, P_{\ell+1}}}}{\abs{P_{\ell+1}}} \leq \frac{\abs{I_{\ell, \SK_{\ell+1}}}}{\abs{\SK_{\ell+1}}} \leq \frac{\abs{O_{\ell+1}}}{\abs{\SK_{\ell+1}}} \, .
	\end{equation}
	Combining both Eq.~\eqref{eq:ml} and  Eq.~\eqref{eq:avgarea}, we can show our claim by
	
	\begin{align}
	&|O_{\ell+1}| \ge \frac{|I_{\ell,P_{\ell+1}}|\,|\SK_{\ell+1}|}{|P_{\ell+1}|} \\
	\iff\,& |O_{\ell+1}|\,|P_{\ell+1}| \ge |I_{\ell,P_{\ell+1}}|\,|\SK_{\ell+1}| \label{eq:step1}\\
	\iff\,& |\SK_{\ell+1}|\,|O_{\ell+1}| + |O_{\ell+1}|\,|P_{\ell+1}|
	\ge |I_{\ell,P_{\ell+1}}|\,|\SK_{\ell+1}| + |\SK_{\ell+1}|\,|O_{\ell+1}| \label{eq:step2}\\
	\iff\,& |O_{\ell+1}|\left(|\SK_{\ell+1}| + |P_{\ell+1}|\right)
	\ge |\SK_{\ell+1}|\left(|I_{\ell,P_{\ell+1}}| + |O_{\ell+1}|\right) \label{eq:step2b}\\
	\iff\,& \frac{|\SK_{\ell+1}|+|P_{\ell+1}|}{|O_{\ell+1}|+|I_{\ell,P_{\ell+1}}|}
	\ge\, \frac{|\SK_{\ell+1}|}{|O_{\ell+1}|} \label{eq:step3}\\
	\overset{\textnormal{\eqref{eq:ml}}}{\iff}& \mathcal{M}_\ell \ge \mathcal{M}_{\ell+1}. \label{eq:step4}
	\end{align}

\end{proof}

\subsubsection{A Note on Densification for Compression}

Note that we have also implemented a densification alternative to this compression scale-space. The algorithm starts with an empty skeleton and greedily adds the skeleton point in every step that has the largest unique impact on the reconstructed object. This yields a similar trade-off between skeleton size and reconstruction error as the sparsification. Quality-wise, both approaches are on par.

However, in this case, densification is inherently a worse choice due to its prohibitive runtime. We implement both sparsification and densification with caching techniques that avoid costly re-computations of the unique influence as much as possible. However, there is an inherent difference between both approaches: In the top-down sparsification approach, we benefit from the fact that in the beginning when the skeleton is highly populated, the local influence is small. As it begins to increase, the number of candidates that need to be updated shrinks. 

In densification, this relationship is reversed. It is much more costly to update the influence of many potential candidates when the skeleton is still sparse and the influence area large. We have observed multiple orders of magnitude higher runtimes for densification in our experiments. Therefore, we do not recommend it in practice and exclude it from our more detailed evaluations.

\begin{figure}[ht]
	\small
	\tabcolsep0pt
	
	\begin{center}
		\begin{tabular}{cccc}
			$1074$ points & $933$ points & $602$ points & $189$ points\\
			\hline
			\multicolumn{4}{c}{(a) random sparsification}\\
			\hline
			\includegraphics[width=0.245\textwidth]{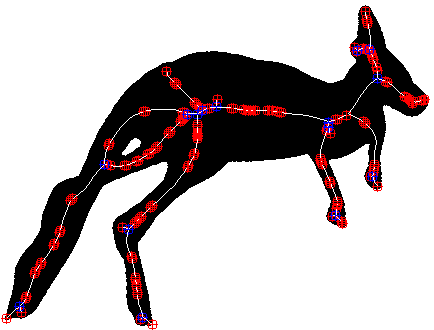} &
			\includegraphics[width=0.245\textwidth]{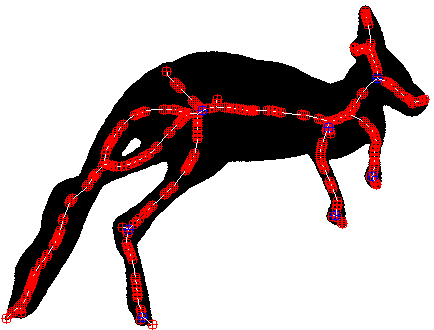} &
			\includegraphics[width=0.245\textwidth]{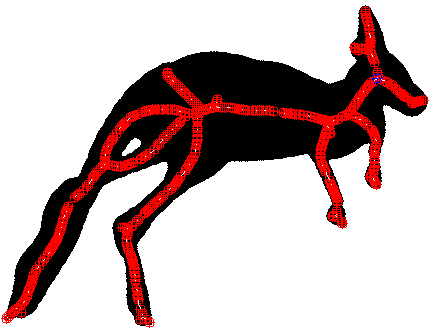} &
			\includegraphics[width=0.245\textwidth]{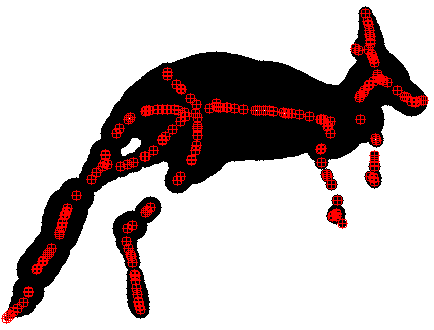} \\
			$\mathcal{C}=157$  & $\mathcal{C}=324$ & $\mathcal{C}=439$ & $\mathcal{C}=184$ \\
			$\mathcal{E}=1883$  & $\mathcal{E}=1948$ & $\mathcal{E}=2208$& $\mathcal{E}=3948$\\
			\hline
			
			\multicolumn{4}{c}{(b) branch pruning}\\
			\hline
			\includegraphics[width=0.245\textwidth]{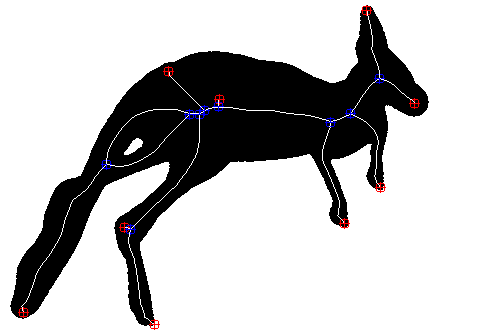} &
			\includegraphics[width=0.245\textwidth]{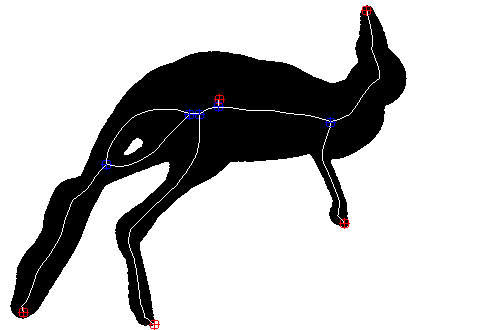} &
			\includegraphics[width=0.245\textwidth]{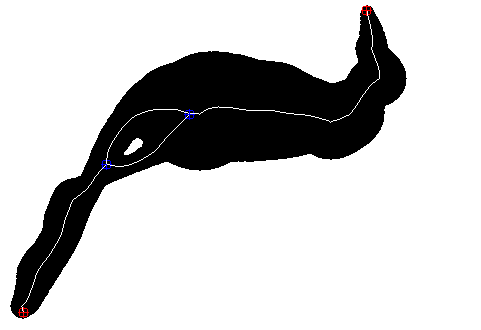} &
			\includegraphics[width=0.245\textwidth]{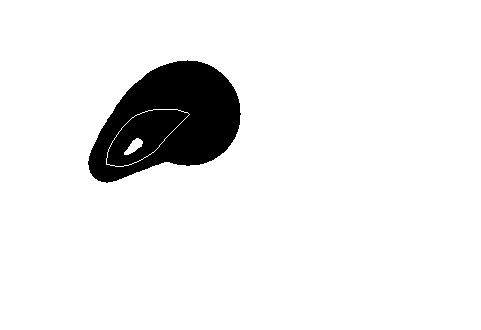} \\
			$\mathcal{C}=18$  & $\mathcal{C}=10$  &  $\mathcal{C}=4$ &  $\mathcal{C}=0$\\
			$\mathcal{E}=1979$  & $\mathcal{E}=3820$ & $\mathcal{E}=9235$& $\mathcal{E}=31702$\\
			\hline

			\multicolumn{4}{c}{(c) compression}\\
			\hline
			\includegraphics[width=0.245\textwidth]{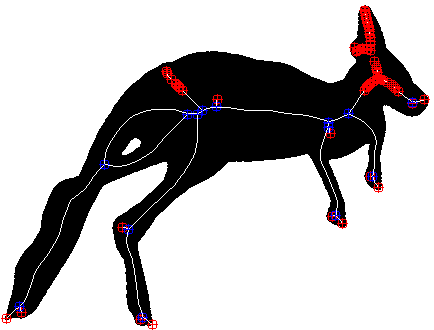} &
			\includegraphics[width=0.245\textwidth]{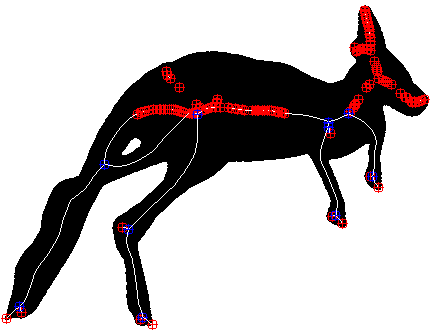} &
			\includegraphics[width=0.245\textwidth]{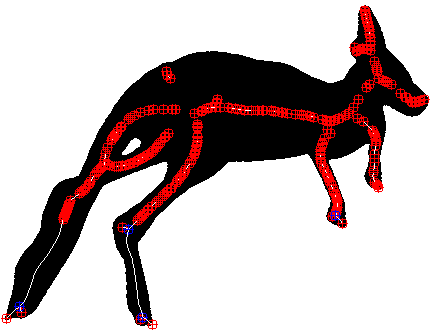} &
			\includegraphics[width=0.245\textwidth]{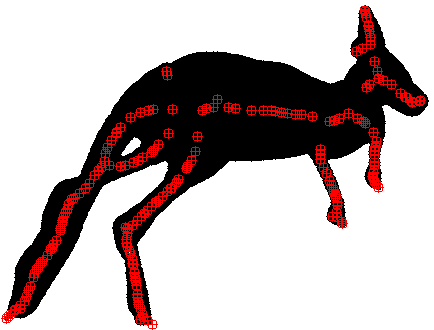} \\
			$\mathcal{C}=76$  & $\mathcal{C}=126$ & $\mathcal{C}=260$ & $\mathcal{C}=188$ \\
			$\mathcal{E}=1825$  & $\mathcal{E}=1825$ & $\mathcal{E}=1825$& $\mathcal{E}=2080$\\
		\end{tabular}
	\end{center}
	\caption{\textbf{Comparison of sparsification paths} on image 69020 of BSDS500 \cite{AMFM11} grouped by branch pruning steps. Endpoints are marked in red, branch points in blue. $\mathcal{E}$ denotes the reconstruction error and $\mathcal{C}$ skeleton complexity. Branch pruning preserves homotopy and reduces complexity at the cost of a higher reconstruction error compared to random or compressive sparsification.\label{fig:branch}}
\end{figure}

\subsection{Comparative Evaluation of Sparsification}

In the following, we compare our branch pruning and compression scale-spaces against the random sparsification baseline and investigate if they fulfil the goals we set for each application. 

\subsubsection{Branch Pruning}

Fig.~\ref{fig:branch} shows a visual example of branch pruning contrasted to random and compressive sparsification. As the example confirms, the topology is preserved as intended. The \emph{kangaroo} object has a hole in it and therefore, the skeleton contains a loop. Since we aim to reduce skeleton complexity, branches with endpoints are prioritised in the pruning: If only a single branch connected to a loop  is left, removing the branch reduces complexity to zero while removing the loop turns one branching point into an end point and keeps minimality constant at two. Both steps would adhere to skeleton minimality as the Lyapunov sequence, but we try to greedily decrease complexity. Thus, as expected, the loop is the last surviving structure of the skeleton. Other branches are pruned one by one according to their impact on the reconstruction. Even with only two branches left, a large part of the shape is retained. 

A clear trade-off can be seen in this example. The preservation of homotopy comes at the price of a higher reconstruction error compared to the random or compression sparsification. On the other hand, skeleton complexity is significantly lower for the branch pruning. For applications like shape matching, these metrics are much more important than an accurate reconstruction of the original shape. Therefore, our scale-space is indeed task-adaptive.

\begin{figure}[p]
	\small
	\tabcolsep0pt
	
	\begin{center}
		\begin{tabular}{cccc}
			$1144$ points & $982$ points & $820$ points & $658$ points\\
			\multicolumn{4}{c}{(a) random sparsification}\\
			\hline
			\includegraphics[width=0.245\textwidth]{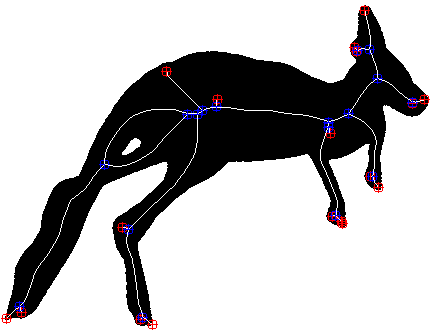} &
			\includegraphics[width=0.245\textwidth]{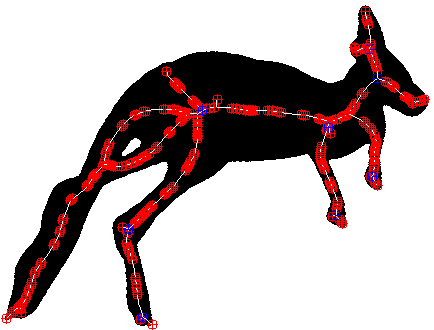} &
			\includegraphics[width=0.245\textwidth]{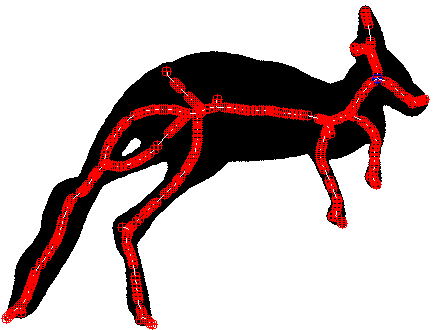} &
			\includegraphics[width=0.245\textwidth]{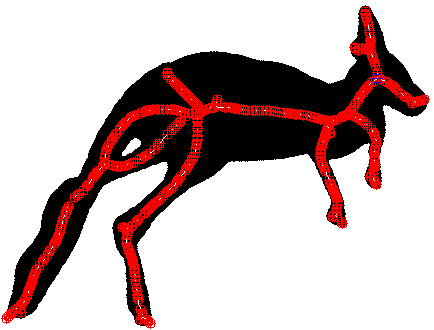} \\
			$\mathcal{E}=1825$  & $\mathcal{E}=1917$ & $\mathcal{E}=1999$& $\mathcal{E}=2149$\\
			\multicolumn{4}{c}{(b) compression}\\
			\hline
			\includegraphics[width=0.245\textwidth]{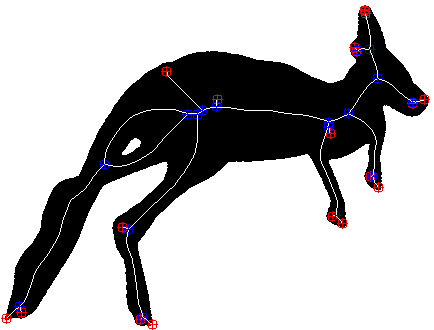} &
			\includegraphics[width=0.245\textwidth]{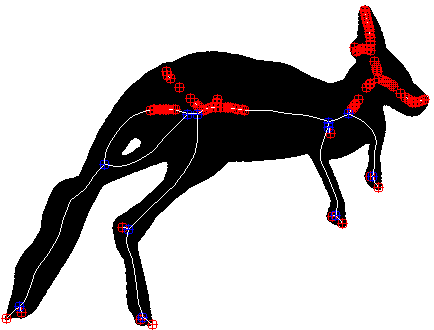} &
			\includegraphics[width=0.245\textwidth]{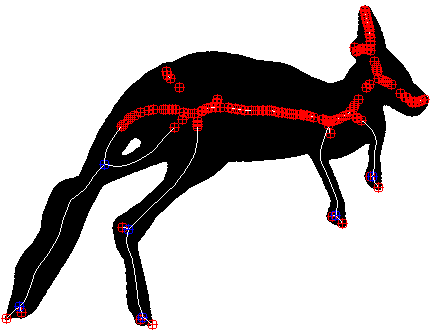} &
			\includegraphics[width=0.245\textwidth]{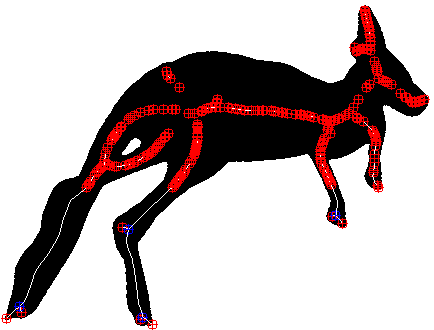} \\
			$\mathcal{E}=1825$  & $\mathcal{E}=1825$ & $\mathcal{E}=1825$& $\mathcal{E}=1825$\\
			
			\hline
			
			$496$ points & $335$ points & $173$ points & $10$ points\\
			\multicolumn{4}{c}{(a) random sparsification}\\
			\hline
			\includegraphics[width=0.245\textwidth]{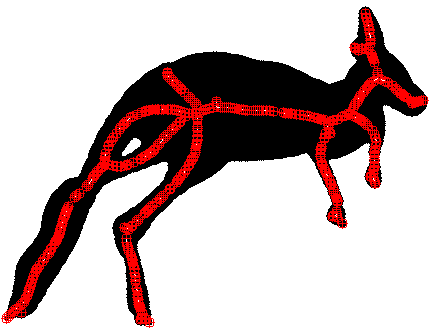} &
			\includegraphics[width=0.245\textwidth]{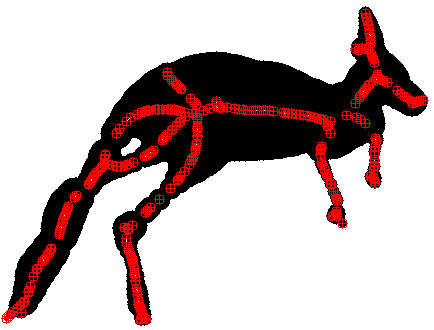} &
			\includegraphics[width=0.245\textwidth]{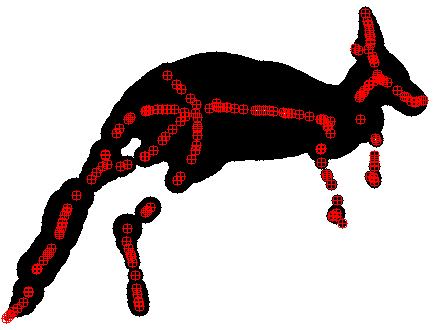} &
			\includegraphics[width=0.245\textwidth]{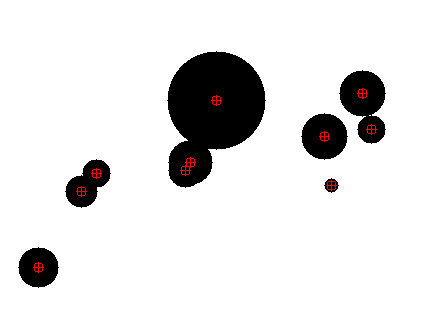} \\
			$\mathcal{E}=2337$  & $\mathcal{E}=2798$ & $\mathcal{E}=4201$& $\mathcal{E}=28266$\\
			
			\multicolumn{4}{c}{(b) compression}\\
			\hline
			\includegraphics[width=0.245\textwidth]{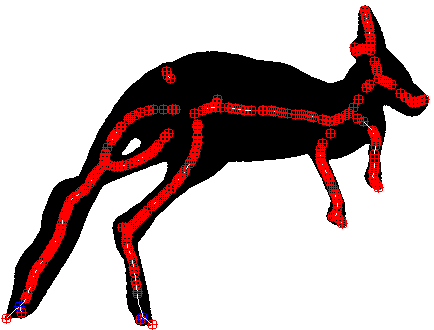} &
			\includegraphics[width=0.245\textwidth]{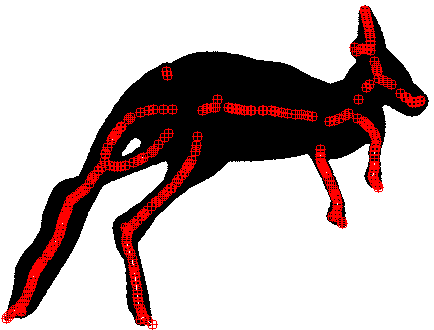} &
			\includegraphics[width=0.245\textwidth]{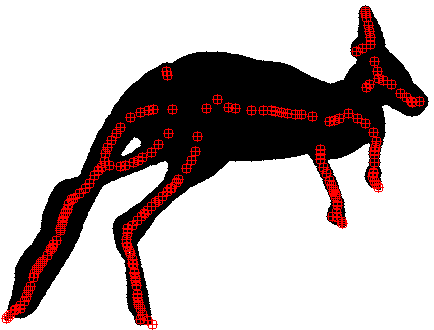} &
			\includegraphics[width=0.245\textwidth]{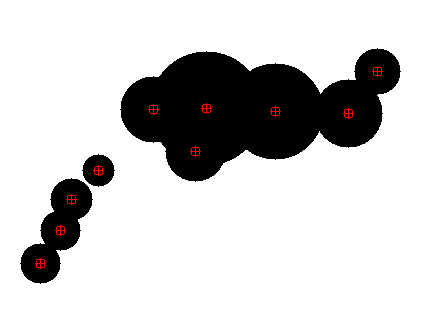} \\
			$\mathcal{E}=1825$  & $\mathcal{E}=1895$ & $\mathcal{E}=2117$& $\mathcal{E}=16574$\\
			\hline
			
		\end{tabular}

	\end{center}
	\caption{\textbf{Comparison of sparsification paths} on image 69020 of BSDS500 \cite{AMFM11} with equidistant scale increase. Compared to random sparsification, the compression scale-space removes only redundant skeleton points first and thus does not increase the reconstruction error $\mathcal{E}$ for the first five scales displayed. In general, it offers a quality improvement of up to 50\%.\label{fig:compression}}
\end{figure}

\subsubsection{Compression}

In Fig.~\ref{fig:branch} and Fig.~\ref{fig:compression}, we see that our compression scale-space outperforms both the sparsification baseline and the branch pruning in terms of reconstruction quality. Here, we are only interested in getting the best trade-off between the number of skeleton points that we need to store and the reconstruction error. Lifting the homotopy requirement and prioritising skeleton minimality over skeleton complexity allows superior performance in this regard. This is also confirmed by the quantitative results in Fig.~\ref{fig:quant}. We can see that compression scale-space lead to an error reduction of up to 80\% compared to branch pruning and up to 50\% compared to the baseline at the same skeleton size.
\begin{figure}[ht]
	\small
	\tabcolsep4pt
	\begin{center}
		\begin{tabular}{cc}
			\includegraphics[height=0.345\textwidth]{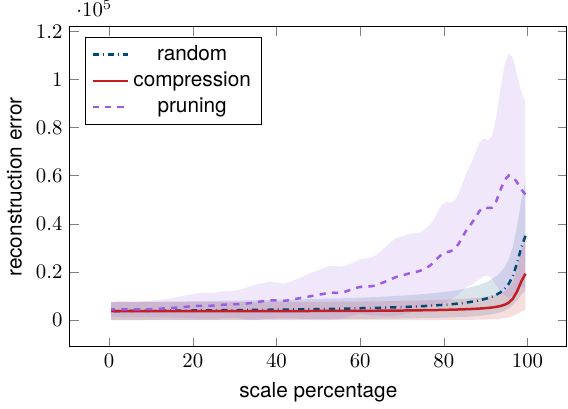} &
			\includegraphics[height=0.325\textwidth]{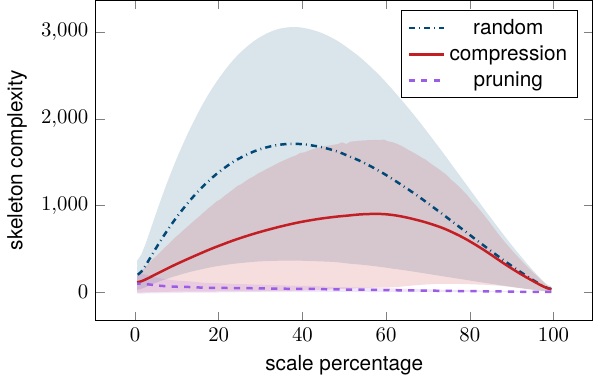} \\
			(a) error & (b) complexity\\
		\end{tabular}

	\end{center}
	\caption{\textbf{Quantitative evaluation} on the 1400 images of the CE-Shape-1 database \cite{latecki2000}. (a) Branch pruning has the highest error due to its preservation of homotopy and reduction of complexity. Compression minimises the reconstruction error and thus performs best. (b) Branch pruning performs best in terms of complexity since it minimises it greedily. Compression implicitly limits the improve in complexity as well since it removes the least significant points first which often form connected branches in the skeleton. The coloured shading represents the $\pm 1$ standard deviation band. All methods us $100$ bins. Due to removal of whole branches these are unevenly populated for pruning. We reduce the noise with a Gaussian filter ($\sigma=1.5$) applied to all methods.\label{fig:quant}}
\end{figure}

\subsubsection{Quantitative Evaluation}

We evaluate reconstruction error $\mathcal{E}$ and skeleton complexity 
$\mathcal{C}$ for our three sparsification scale-spaces on the CE-Shape-1 
database \cite{latecki2000}. Since the objects and their associated skeletons 
have highly variable sizes, the number of discrete steps in scale-space varies 
as well. Therefore, we cannot average errors over the discrete scale. Instead, 
we normalise scale to a $0-100$ percentage scale.

The evolution of the error in Fig.~\ref{fig:quant}(a) is consistent with the visual evaluation from the previous sections. Branch pruning trades homotopy and monotonically decreasing skeleton complexity for a relatively high reconstruction error. In contrast, the greedy optimisation of our compression scale-space guarantees the minimum increase in the error for every step in scale-space. This is verified by its quantitative performance. It is particularly effective for very coarse scales, where the choices of the preserved skeleton points matter most.

In the case of pruning, the graph of the skeleton complexity in Fig.~\ref{fig:quant}(b) verifies the expectations that we have for a Lyapunov sequence. It is a meaningful measure for shape matching and declines monotonically. Both compression and random sparsification feature an initial increase in complexity followed by a decline for coarse scales. Since they are not homotopy-preserving, both of these scale-spaces disconnect the skeleton early on, thus creating new end points. However, random sparsification does this without distinction of skeleton points, thus leading to a slightly distorted bell curve. The increase and decline of complexity for compression however is linear. This arises from the fact that there is a certain correlation between reconstruction impact and skeleton complexity: Compression removes the skeleton points with the least impact first. Neighbouring skeleton points often have similar disk radii as well, thus their reconstruction disks are heavily overlapping. Therefore, if a skeleton point is insignificant, there is a high chance of its neighbours also being insignificant. Thus, it is likely that whole connected branches are removed first.

Overall, the quantitative evaluation supports our visual observations.

%\afterpage{\FloatBarrier}
%Runtimes: 
%12s random, branch, compression, densification: >34min.

\subsection{Stiffness Enhancement with Extended Densification Scale-Spaces}
\label{sec:stiffexp}

\begin{algorithm2e}[t]
	\caption{Stiffness Enhancement \label{alg:stiff}}
	\tcc{Initialisation}
	$\ell \gets 0$\;
	$A_0 \gets O$\;
	$\SE_0 \gets \SK$\;
	\While{$\abs{\SE_\ell} < \abs{O}$}{
		\tcc{Strengthen current skeleton by dilation.}
		\tcc{$\mathcal{N}(\bm x)$: 8-neighbourhood of $\bm x$}
		$\hat{\SE_\ell} \gets \SE_\ell \cup \{\mathcal{N}(\bm x) \, \vertL \, \bm x \in \SE_\ell\}$\;
		\tcc{Update auxiliary shape by subtracting dilated skeleton.}
		$A_{\ell+1} \gets A_\ell \setminus \hat{\SE_\ell}$\;
		\tcc{Add skeleton of current auxiliary shape $A_\ell$ to evolving generalised skeleton.}
		$\SE_{\ell+1} \gets \hat{\SE_\ell} \cup \ST{A_\ell}$\;
		$\ell \gets \ell+1$\; 	
	}
\end{algorithm2e}

In the following, we demonstrate that our extended densification from Section~\ref{sec:densext} allows us to perform stiffening \cite{BBYBP20}. Given the outline of an object to be printed with a 3-D printer, an iterative skeletonisation can improve the stability of the printed object. They have validated that such an approach actually improves the elastic stress properties. 

Our highly generalised Definition~\ref{def:extpath} is flexible enough to cover applications that go far beyond a standard skeletonisation. In Algorithm~\ref{alg:stiff}, our first iteration simply adds the original skeleton of the shape to our extended evolving skeleton set $\SE_\ell$. This forms the basis for the printed structures. Then, in every step, we first strengthen the pre-existing structure by a one pixel dilation: We thicken the skeleton by adding all direct non-skeleton neighbours. This is a step that highlights the generalisation capabilities of our framework: We can intentionally sacrifice the thin skeleton property in order to adopt the stiffness enhancement idea that relies on thickening main support structures. Then, we subtract this skeleton from the object, creating an evolving auxiliary object $A_\ell$. To this auxiliary object, we apply the skeletonisation algorithm again. By repeating these steps, in the transition from every coarse to the next finer scale, we strengthen the existing support structures and add new thin support beams given by the newly computed skeleton.

\begin{figure}[t]
	\small
	\tabcolsep0pt
	
	\begin{center}
		\begin{tabular}{cccc}
			\multicolumn{4}{c}{(a) bat-1}\\
			\hline
			\includegraphics[width=0.245\textwidth]{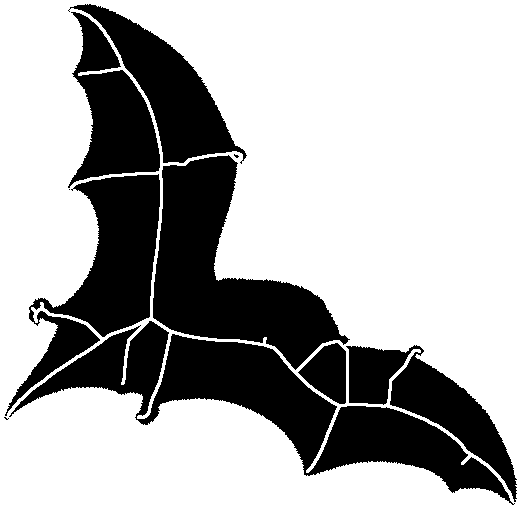} &
			\includegraphics[width=0.245\textwidth]{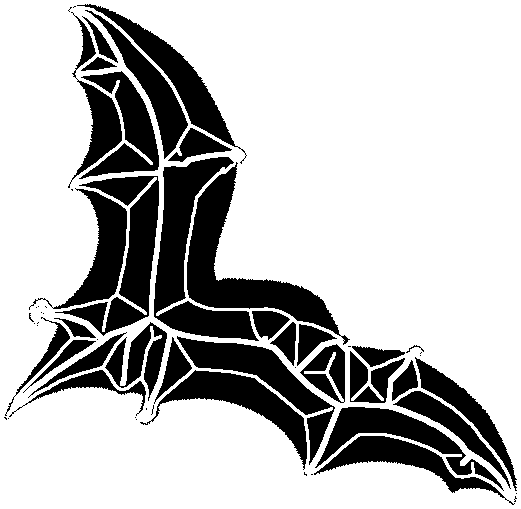} &
			\includegraphics[width=0.245\textwidth]{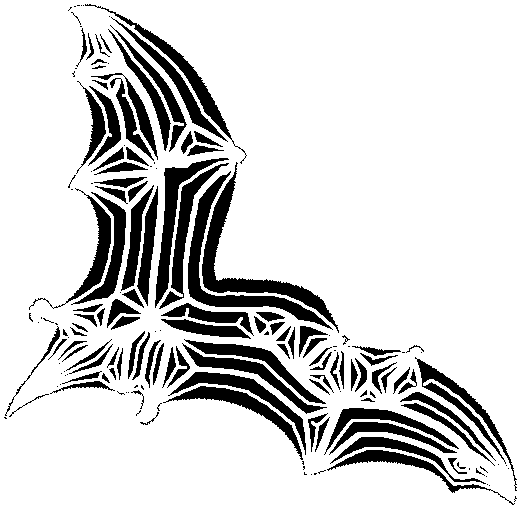} &
			\includegraphics[width=0.245\textwidth]{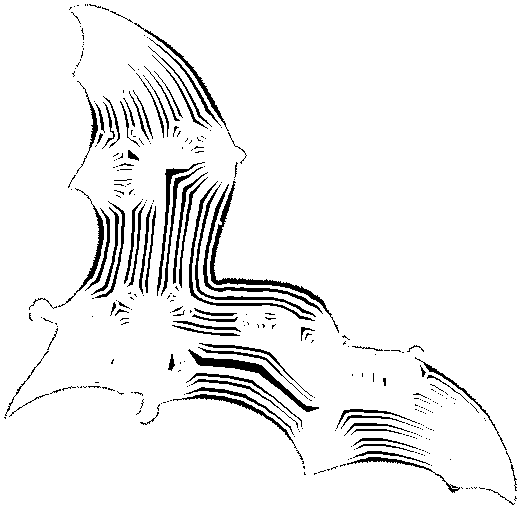} \\
			\hline
		& & & \\
			
			\multicolumn{4}{c}{(a) device8-3}\\
			\hline
			\includegraphics[width=0.245\textwidth]{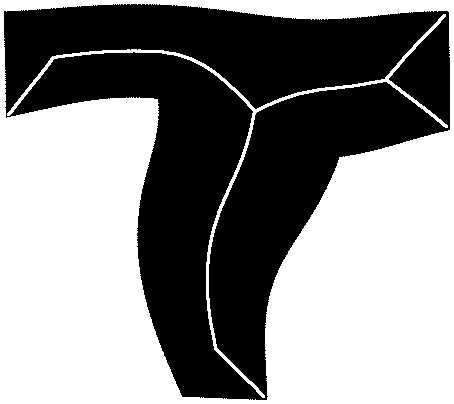} &
			\includegraphics[width=0.245\textwidth]{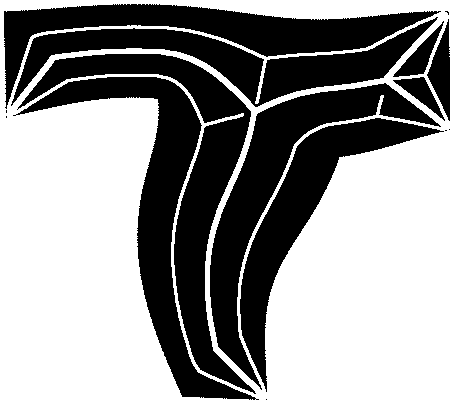} &
			\includegraphics[width=0.245\textwidth]{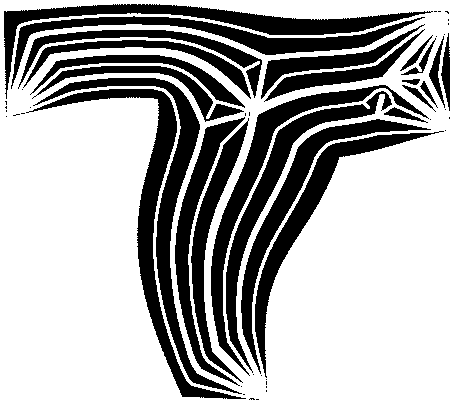} &
			\includegraphics[width=0.245\textwidth]{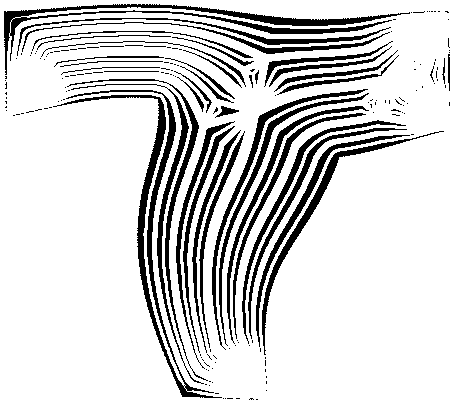} \\
			\hline		
		\end{tabular}

	\end{center}
	\caption{\textbf{Stiffness enhancement} on two images of the CE-Shape-1 database \cite{latecki2000}. Each finer scale ads more branches to the extended skeleton and thickens the previous skeleton.\label{fig:stiff}}
\end{figure}

Fig.~\ref{fig:stiff} shows two examples of stiffness enhancement. As intended, we populate the extended skeleton with new support branches and iteratively strengthen already existing branches.

\afterpage{\FloatBarrier}

\section{Conclusions and Outlook}
\label{sec:conclusion}

We leverage a combination of data optimisation strategies \cite{CPW19,Pe21,MHWT12,HMHW17,DDI06,CW21,DAW21,KBPW18} and the medial axis transform \cite{Bl67} to establish a new class of scale-space frameworks for shape analysis. With only a few requirements, they offer a broad range of architectural and simplification properties. Compared to many other scale-spaces, they have no ill-posed direction: sparsification and densification scale-spaces are equally well-posed counterparts that allow two approaches for transitioning between coarse and fine shape representations in opposite directions. Together with the sparsification and densification paths as flexible design tools, our framework allows a wide range of practical applications. With proof-of-concept examples from the areas of skeleton post-processing, shape compression, and 3-D printing, we have demonstrated how additional requirements for the scale-space paths yield task-specific theoretical guarantees.

In the future, we plan to explore further practical applications. In particular, we want to investigate scale-spaces for compression that also take coding cost into account, similar to quantisation scale-spaces \cite{Pe21}.

\backmatter

%\bmhead{Supplementary information}
%
%If your article has accompanying supplementary file/s please state so here. 
%
%Authors reporting data from electrophoretic gels and blots should supply the full unprocessed scans for key as part of their Supplementary information. This may be requested by the editorial team/s if it is missing.
%
%Please refer to Journal-level guidance for any specific requirements.
%
%\bmhead{Acknowledgements}
%
%Acknowledgements are not compulsory. Where included they should be brief. Grant or contribution numbers may be acknowledged.
%
%Please refer to Journal-level guidance for any specific requirements.

\section*{Declarations}

\subsection{Author Contribution} 

Both authors have contributed equally to both the theoretical content and practical implementations.

\subsection{Competing Interests}

No competing interests exist. 
%
%
%
%Some journals require declarations to be submitted in a standardised format. Please check the Instructions for Authors of the journal to which you are submitting to see if you need to complete this section. If yes, your manuscript must contain the following sections under the heading `Declarations':
%
%\begin{itemize}
%\item Conflict of interest/Competing interests (check journal-specific guidelines for which heading to use)
%\item Ethics approval and consent to participate
%\item Consent for publication
%\item Data availability 
%\end{itemize}
%
%\noindent
%If any of the sections are not relevant to your manuscript, please include the heading and write `Not applicable' for that section. 
%
%%%===================================================%%
%%% For presentation purpose, we have included        %%
%%% \bigskip command. Please ignore this.             %%
%%%===================================================%%
%\bigskip
%\bigskip\noindent
%Springer journals and proceedings: \url{https://www.springer.com/gp/editorial-policies}
%
%\bigskip\noindent
%Nature Portfolio journals: \url{https://www.nature.com/nature-research/editorial-policies}
%
%\bigskip\noindent
%\textit{Scientific Reports}: \url{https://www.nature.com/srep/journal-policies/editorial-policies}
%
%\bigskip\noindent
%BMC journals: \url{https://www.biomedcentral.com/getpublished/editorial-policies}
%\end{flushleft}

\begin{appendices}

\section{Maximum Disk Thinning}
\label{app:maxdisc}

For the discrete experiments in this paper we employ the maximal disk thinning (MDT) 
to compute the initial skeletons. As the concepts of the object $O$, the
distance transform $D$, and homotopic thinning have already been introduced
in Section~\ref{sec:skeletonisation}, we only summarise the algorithmic procedure
here. A detailed description and evaluation can be found in
\cite{PB12}.

The method extends upon the original approach by \citet{remy2005} which proposes to consider the centres of inscribed disks with maximal radius as candidates for the skeleton. However, MDT also incorporates the original grass-fire idea of \citet{Bl67} in a level set paradigm: From the boundary inwards, it considers sets of points that have equal values in the distance transform $D$ (i.e. $D$ is increasing with each iteration).

All object points that are not skeleton candidates are removed in each iteration, unless their removal destroys the homotopy of the remaining shape. To avoid spurious branches, isolated skeleton candidates are pruned if they have no neighbouring candidates in a $5 \times 5$ window.

Algorithm~\ref{alg:mdt} contains the pseudocode of the MDT algorithm adapted from \cite{PB12}. For the definition of simple and end points, please refer to Definition~\ref{def:points} the helper function \emph{isSimple} and \emph{isEndpoint} check the criteria of this definition. Fig.~\ref{fig:thin} illustrates MDT thinning on a toy example.

\begin{algorithm2e}[t]
	\caption{Maximal Disk Thinning\label{alg:mdt}}
	\tcc{Inputs (precomputed): $order$ (cached ordering values according to $D$), $ep$ (endpoints identified with the RT algorithm)}
	\tcc{$hp$: min-heap ordered by $order[\cdot]$, $N_8,N_{24}$: 8- and $5\times5$-neighbourhoods}
	
	\tcc{Prune spurious endpoint candidates.}
	\ForEach{$p \in O$}{
		\If{$\bigl|\{x \in N_{24}(p)\, \vertL \, ep[x]=\textnormal{TRUE}\}\bigr| < 2$}{
			$ep[p] \gets \textnormal{FALSE}$\;
		}
	}
	
	\tcc{Initialise heap with boundary points marked for deletion.}
	\ForEach{$p \in \partial O$}{
		\If{$isSimple(p)$}{
			insert $p$ into $hp$ with key $order[p]$\;
		}
	}
	
	\tcc{Homotopy preserving thinning with maximal-disk endpoint protection}
	\While{$|hp|>0$}{
		$p \gets popmin(hp)$\;
		\If{$isSimple(p)$ \textbf{and} $\bigl(\neg isEndpoint(p)\ \textbf{or}\ ep[p]=\textnormal{FALSE}\bigr)$}{
			remove $p$ from $O$\;
			\ForEach{$q \in N_8(p)$}{
				\If{$isSimple(q)$}{
					insert $q$ into $hp$ with key $order[q]$\;
				}
			}
		}
	}
\end{algorithm2e}

\begin{figure}[t]
	\centering
	\begin{tikzpicture}[
		x=0.28cm, y=0.28cm,
		cell/.style={draw=black!30, line width=0.2pt},
		pix/.style={fill=black, draw=black!30, line width=0.2pt},
		arr/.style={->, line width=0.6pt, draw=black!55}
		]
		
		\def\W{9}
		\def\H{5} 
		
		\newcommand{\StageList}[2]{%
			\begin{scope}[shift={(#1,0)}]
				% grid
				\foreach \i in {0,...,\W}{
					\draw[cell] (\i,0) -- (\i,\H);
				}
				\foreach \j in {0,...,\H}{
					\draw[cell] (0,\j) -- (\W,\j);
				}
				% pixels
				\foreach \x/\y in {#2}{
					\path[pix] (\x,\y) rectangle ++(1,1);
				}
			\end{scope}
		}
		
		\newcommand{\StageRect}[5]{%
			\begin{scope}[shift={(#1,0)}]
				\foreach \i in {0,...,\W}{
					\draw[cell] (\i,0) -- (\i,\H);
				}
				\foreach \j in {0,...,\H}{
					\draw[cell] (0,\j) -- (\W,\j);
				}
				\foreach \x in {#2,...,#3}{
					\foreach \y in {#4,...,#5}{
						\path[pix] (\x,\y) rectangle ++(1,1);
					}
				}
			\end{scope}
		}

		% (1) full rectangle (9x5)
		\StageRect{0}{0}{8}{0}{4}
		
		%		
		% (2) first pruning step
		\StageList{13}{
			0/0, 0/4, 8/0, 8/4,
			1/1, 2/1, 3/1, 4/1, 5/1, 6/1,
			1/2, 2/2, 7/2, 
			1/3, 2/3, 3/3, 4/3, 5/3, 6/3,
			7/1, 6/2, 7/3,
			3/2,4/2,5/2
		}
		
		% (3) final skeleton
		\StageList{26}{
			0/0, 0/4, 8/0, 8/4,
			1/1, 2/2, 1/3,
			7/1, 6/2, 7/3,
			3/2,4/2,5/2
		}
		
		% ----- arrows between stages ---------------------------------------------
		\draw[arr] (10.2,2.5) -- (12.2,2.5);
		\draw[arr] (23.2,2.5) -- (25.2,2.5);

	\end{tikzpicture}
	
	\medskip
	
	\caption{\textbf{Toy illustration of thinning on a rectangle.}
		The full rectangle of step 1 is reduced to its thin discrete skeleton in two steps. In each step, we remove one level set of non-skeleton points as in the MDT algorithm.}
	\label{fig:thin}
\end{figure}

%
%\section{Section title of first appendix}\label{secA1}
%
%An appendix contains supplementary information that is not an essential part of the text itself but which may be helpful in providing a more comprehensive understanding of the research problem or it is information that is too cumbersome to be included in the body of the paper.




\end{appendices}


\bibliography{peter_refs.bib}% common bib file

\end{document}